\documentclass[twoside]{article}

\usepackage[accepted]{aistats2024}

\usepackage{amsmath}
\usepackage{amssymb}
\usepackage{mathtools}
\usepackage{amsthm}
\usepackage{dsfont}

\usepackage{accents}

\usepackage{notations/ISG_style}
\usepackage{notations/ISG_style_reduced}

\usepackage{enumerate}
\theoremstyle{plain}
\newtheorem{theorem}{Theorem}[section]

\newtheorem{lemma}[theorem]{Lemma}
\newtheorem{corollary}[theorem]{Corollary}
\theoremstyle{definition}
\newtheorem{definition}[theorem]{Definition}
\newtheorem{assumption}[theorem]{Assumption}

\theoremstyle{remark}

\newcommand{\kl}{\;\|\;}
\newcommand{\f}{\; ; \;}
\newcommand{\cn}{{\rm cn}}
\usepackage{float}

\newcommand{\mathleft}{\@fleqntrue\@mathmargin0pt}
\newcommand{\mathcenter}{\@fleqnfalse}

\usepackage{algorithm}
\usepackage{algpseudocode}

\usepackage[dvipsnames]{xcolor}       

\usepackage[hidelinks]{hyperref}
\usepackage[round]{natbib}

\usepackage{makecell}

\usepackage[toc,page,header]{appendix}
\usepackage{minitoc}

\begin{document}

%

%

\twocolumn[

\aistatstitle{Causal Bandits with General Causal Models and Interventions}

\aistatsauthor{Zirui Yan$^* \quad $  Dennis Wei$^\dagger \quad$ Dmitriy Katz-Rogozhnikov$^\dagger \quad$ Prasanna Sattigeri$^\dagger \quad $ Ali Tajer$^*$ }

\aistatsaddress{$^*$Rensselaer Polytechnic Institute \And $^\dagger$IBM Research} ]

\begin{abstract}
This paper considers causal bandits (CBs) for the sequential design of interventions in a causal system. The objective is to optimize a reward function via minimizing a measure of cumulative regret with respect to the best sequence of interventions in hindsight. The paper advances the results on CBs in three directions. First, the structural causal models (SCMs) are assumed to be \emph{unknown} and drawn arbitrarily from a general class $\mcF$ of Lipschitz-continuous functions. Existing results are often focused on (generalized) linear SCMs. Second, the interventions are assumed to be \emph{generalized soft} with any desired level of granularity, resulting in an infinite number of possible interventions. The existing literature, in contrast, generally adopts \emph{atomic} and \emph{hard} interventions. Third, we provide general upper \emph{and} lower bounds on regret. The upper bounds subsume (and improve) known bounds for special cases. The lower bounds are generally hitherto unknown. These bounds are characterized as functions of the (i)~graph parameters, (ii)~eluder dimension of the space of SCMs, denoted by $\operatorname{dim}(\mcF)$, and (iii)~the covering number of the function space, denoted by $\cn(\mcF)$. 
Specifically, the cumulative achievable regret over horizon $T$ is $\mcO(K d^{L-1}\sqrt{T\operatorname{dim}(\mcF) \log(\cn(\mcF))})$, where $K$ is related to the Lipschitz constants, $d$ is the graph’s maximum in-degree, and $L$ is the length of the longest causal path. The upper bound is further refined for special classes of SCMs (neural network, polynomial, and linear), and their corresponding lower bounds are provided.
\end{abstract}


\section{INTRODUCTION}
Causal Bayesian networks are probabilistic graphical models that formalize the cause-effect relationships within a network of interacting entities. The graph nodes represent the entities, and directed edges encode their statistical cause-effect dependencies. Structural causal models (SCMs) further specify structural equations that quantify the cause-effect relationships. In the analysis of causal networks, \emph{interventions} are the key mechanism for uncovering cause-effect relationships. Interventions are external treatments that induce variations in the Bayesian network's structure or its probabilistic models. Consequently, experimental design in causal networks naturally hinges on designing interventions, which
involves two dimensions: (i) selecting the graph nodes for intervening and (ii) specifying the appropriate interventions (changes) on the selected nodes. Intervening at a node can be as simple as forcing the random variables associated with it to take specified fixed values (hard intervention) or as complex as modifying the statistical model between the node and its causal parents in the graph (soft interventions).

\emph{Causal bandits}, introduced by~\citet{lattimore2016causal}, provide an effective framework for formalizing and analyzing the sequential design of experiments in causal networks. The objective of causal bandits is to design a sequence of interventions that can optimize an inferential task over the network. Causal bandits are so-called because the interventions are modeled as bandit \emph{arms} and the inference quality is modeled as the bandit \emph{reward}. Causal bandits have a wide range of applications in  robotics~\citep{ahmed2021causalworld}, gene interaction networks~\citep{badsha2019learning} , drug discovery~\citep{liu2020reinforcement}, advertising platforms~\citep{zhao2022mitigating}, and the broader economic domain.

\begin{table*}[htb]
\label{tab:bounds1}
    \centering
    \begin{tabular}{|c|l|l|l|}
        \hline
         SCM class & Upper bounds & Class effect & Lower bounds  \\ \hline \hline
         General & $  \mcO\left(K d^{L-1} \sqrt{T\log T\; \varpi_\mcF} \right)$ & $\varpi_\mcF = \operatorname{dim}(\mcF)\left[1+\frac{\log N+\log \cn(\mcF)}{\log T}\right] $
                &  -- 
                \\ \hline
         Linear & $   \mcO\Big(Kd^{L}\sqrt{T\log T\; \varpi_{\rm L}} )\Big)$ & $\varpi_{\rm L} = \log T \left[1+\frac{\log N}{d\log T} \right]$
                & $\Omega\left(Kd^{\frac{L}{2}-1}\sqrt{T}\right)$\\ \hline 
         Polynomial& $  \mcO \Big(K d^{L+1}\sqrt{T\log T\; \varpi_{\rm P}}\Big)$ for $p=2$ & $\varpi_{\rm P} = \log T \left[1+\frac{\log N}{d\log T} \right]$ 
         & $\Omega\left(K \sqrt{T}\right)$
         \\ \hline
         \makecell{Neural network\\(width $s\geq d$)} & $  \mcO\Big(\bar Kd^{L-1}\sqrt{T\log T\; \varpi_{\rm N}}\Big)$ & $\varpi_{\rm N} = 
         s \log T \left[1+\frac{\log N}{d\log T} \right]$
                & $\Omega\left(K d^{\frac{L}{2}-1}\sqrt{T}\right)$ \\        \hline 
    \end{tabular}
    \caption{Comparison of upper and lower bounds on regret for general SCMs and special cases.}
    \vspace{0.1 in}
\end{table*}

\noindent\textbf{Contributions.} We consider a given causal graph with (i) \emph{unknown} SCMs, (ii) \emph{general soft} intervention models on the graph nodes, and SCMs with an arbitrary level of granularity. Following the causal bandit framework, the objective is the sequential and data-adaptive design of interventions to optimize an \emph{unknown} reward function. We propose the general causal bandit (GCB) algorithm to design interventions in this general setting. Lower and upper bounds are established for the cumulative regret of the algorithm, that is, the regret accumulated over time with respect to the best intervention in hindsight. The regret results are subsequently specialized and refined for three main classes of SCMs (linear, polynomial, and neural networks).

\noindent\textbf{Key observations.} The SCMs are assumed to belong to general classes of functions $\mcF$ (with mild regularity conditions). We provide an overview below of how the regret bounds are affected by graph topology, SCMs, and intervention models. The regret bounds are summarized in Table~\ref{tab:bounds1}. \vspace{-.1 in}
\begin{itemize}
    \item \textbf{SCM class:} The achievable regret depends on two intrinsic measures associated with the SCM class $\mcF$, namely the eluder dimension of $\mcF$ ($\operatorname{dim}(\mcF)$) and the covering number of $\mcF$ ($\cn(\mcF)$). 
    \vspace{-0.15 in}
    \item \textbf{Graph topology:} The regret lower and upper bounds have diminishing dependence on the graph size $N$ and instead depend on the graph through two key parameters: the maximum in-degree ($d$) and the length of the longest causal path ($L$).  
    \vspace{-0.05 in}
    \item \textbf{Intervention granularity:} The regret lower and upper bounds are \emph{independent} of the cardinality of the parameters specifying the interventions. 
    \vspace{-0.05 in}
    \item \textbf{Lower bound:} A general minimax lower bound on regret is given without restrictions on $\mcF$. This lower bound is further refined for special classes of functions.
 \item \textbf{Dependence on horizon:} The regret lower and upper bounds both scale with the time horizon $T$ with the rate $\tilde \mcO(\sqrt{T})$.
    \vspace{-0.05 in}
\end{itemize}

\noindent\textbf{Relevant literature.} The existing literature on causal bandits can be categorized along four dimensions: (i)~whether the causal graph is known or unknown, (ii)~whether the pre- and post-intervention probability distributions of the causal graph are known or unknown, (iii)~the class of SCMs adopted, and  (iv)~the nature of interventions (e.g., granularity). Most existing studies focus on the first two dimensions, with the last two aspects highly under-investigated. Specifically, based on the availability of the causal graph and the interventional distributions, there are broadly four distinct categories of studies. The initial study by \citet{lattimore2016causal} assumes knowledge of both the causal graph and all the interventional distributions. Similar assumptions are adopted by \citet{bareinboim2015bandits,sen17,lu2020regret,nair2021budgeted}. There has been progress towards relaxing these assumptions. Specifically, the study by \citet{bilodeau2022adaptively} assumes that the causal graph structure is unknown while the interventional distributions are known.
In contrast, settings in which the causal graph is known but the interventional distributions are unknown are studied by \citet{yabe18causal,maiti2022causal,xiong2022pure,feng2023combinatorial,sawarni2023learning}. Finally, the case of unknown graph structure and unknown interventional distributions is investigated by \citet{lu2021causal,de2022causal,bilodeau2022adaptively,feng2023combinatorialwioutgraph,konobeev2023causal,malek2023additive}.

These studies, generally, focus on idealized models that involve (i)~ideal \emph{do} interventions, which completely remove parental causal effects and can be restrictive in many applications, and (ii)~simple classes of structural causal models. The SCMs investigated include linear  \citep{maiti2022causal} or binary generalized linear functions \citep{xiong2022pure,jaber2020causal}. 

Two recent studies by~\ \citet{varici2022causal,sussex2022model} consider $soft$ interventions, which, instead of cutting off parental causal effects, only modify the conditional distribution of the nodes. The study by \citet{varici2022causal} assumes a \emph{linear} SCM with a \emph{binary} intervention model per node. In contrast, we assume \emph{general} SCMs with an arbitrary \emph{continuum} of interventions. Furthermore, our general model can be specialized to linear SCMs and in this case, our achievable regret bound improves on that of~\citep{varici2022causal}.

The study by~\citet{sussex2022model} investigates the class of \emph{Gaussian processes} within a reproducing kernel Hilbert space (RKHS), assumes controllable variables for specifying interventions, and characterizes achievable regret bounds. In contrast, we focus on general SCMs, and our achievable regret has two significant distinctions. First, the achievable regret of~\citet{sussex2022model} scales \emph{linearly} with the graph size $N$. In contrast, in the most general form, our regret scales with $\sqrt{\log N}$, and more importantly, for most special cases, it scales with $\sqrt{\frac{\log N}{\log T}}$, implying a diminishing dependence on $N$ as the horizon $T$ grows. Furthermore, the regret upper bound in~\citet{sussex2022model} relies on the maximum information gain~\citep{srinivas2009gaussian}. This renders an achievable regret in which all model parameters grow exponentially with $L$. As a result, the associated cumulative regret bound loses its sublinear growth guarantee in $T$ in some models, e.g., Gaussian processes under bounded parameter RKHS with polynomial spectral decay  (see Appendix~\ref{sec:compareGP} for more details). Furthermore, besides the achievable bound, we also characterize a general minimax lower bound and its refined counterparts for special cases. 

\noindent\textbf{Notations.}
For a positive integer $N\in \N$, we define $[N]\triangleq \{1,\dots,N\}$. The Euclidean norm of vector $X\in\R^d$ is denoted by $\norm{X}$. For a function $f$ and a set of its arguments $\mcZ$,  we denote the empirical quadratic error of $f$ with respect to the set $\mcZ$ by $\norm{f}_{\mcZ}^2 \triangleq  \sum_{z\in\mcZ} f^2(z)$. The infinity norm of function $f:\mcZ\to\R$ is denoted by $\norm{f}_{\infty} \triangleq \sup_{z\in\mcZ}|f(z)|$. For a set $\mcA$, we use $|\mcA|$ to denote the cardinality of the set.

\section{GCB MODEL AND OBJECTIVE}
\label{sec:GCB}
\subsection{Causal Graph with General SCM} 
Consider a directed acyclic graph (DAG) $\mathcal{G}(\mathcal{V},\mathcal{E})$, in which $\mathcal{V}=[N]$ is the set of nodes, and $\mathcal{E}$ denotes the set of directed edges. The ordered tuple $(i,j)\in\mathcal{E}$ indicates a directed edge from node $i\in [N]$ to node $j\in [N]$.  We refer to the parents of node $i\in [N]$ by $\Pa(i)$. We define the \emph{causal depth} of node $i$ as the length of the longest directed causal path that ends at node $i\in[N]$ and denote it by $L_i$. We define $d_i\triangleq |\Pa(i)|$ as the in-degree of node $i$ and the maximum degree in $\mcG$ is denoted by $d\triangleq \max_{i\in[N]}d_i$, and the length of the longest directed causal path in $\mcG$ is denoted by 
$L=\max_{i\in[N]}L_i$. We define $\bX\triangleq (X_1,\dots, X_n)^\top$ as the vector of random variables associated with the nodes. We consider the following general SCM:
\begin{equation}
\label{eq:SEM_1}
X_{i}  \; =\; f_i(X_{\Pa(i)})\; + \; \epsilon_i \ , \quad \forall i\in[N] \ ,
\end{equation}
where functions $\{f_i: i \in[N]\}$ are \textbf{unknown} and $f_i$ belongs to a specified class of functions denoted by $\mcF_i$ (e.g., linear, polynomial, or neural network). Furthermore, our analysis is sometimes specialized to \emph{parameterized} (i.e., finite-dimensional) families of distributions. In such cases, we denote the parameters of model $f_i$ by~$\btheta_i$. The terms $\{\epsilon_i: i \in[N]\}$ account for additive noise terms, according to which we define the noise vector $\bepsilon\triangleq (\epsilon_1,\epsilon_2,\cdots, \epsilon_N)^\top$. Noise terms are independent across nodes.  Our study centers on the causal bandit problem over a \emph{known} DAG.

\subsection{Generalised Soft Interventions}
The conditional distribution of $X_i$ given its parents, i.e., $\P(X_i\med X_{\Pa (i)})$, is assumed to be \textbf{unknown}. A soft intervention on node $i$ changes this conditional distribution. We consider a \emph{generalized} soft intervention mechanism such that intervening on node $i$ changes $\P(X_i\med X_{\Pa (i)})$ to a \emph{class of} distinct distributions. Interventions on node $i$ are parameterized by $a_i\in\mcA_i\subseteq \R$, and the intervention spaces $\{\mcA_i:i\in[N]\}$ can be discrete or a continuum of models.  This generalized model subsumes the widely used atomic intervention models. To formalize the effect of interventions, we incorporate their impact as changes on the functions $\{f_i:i\in[N]\}$. Specifically, under an interventional on node $i$ parameterized by $a_i$, function $f_i\in\mcF_i$ changes to another function in the same class $\mcF_i$, which we denote by $f_i(\cdot ; a_i)$. Hence, the interventional counterpart of the SEMs in~\eqref{eq:SEM_1} under an intervention parameterized by $a_i\in\mcA_i$ becomes
\begin{equation}
\label{eq:SCM}
X_{i} \; = \; f_i(X_{\Pa(i)} \f a_i)+\epsilon_i \ , \quad \forall i\in[N] \ .
\end{equation}
To unify the observational and interventional model in~\eqref{eq:SEM_1}~and~\eqref{eq:SCM}, without loss of generality, we use the convention that setting $a_i=0$ in~\eqref{eq:SCM} generates the observational model.
We allow simultaneous intervention at any desired set of nodes. Hence, an intervention is specified by the intervention vector $\ba \triangleq (a_1,\cdots, a_N)^\top$, generated from the intervention space  $\mcA\triangleq \prod_{i=1}^N \mcA_i$. Under intervention $\ba\in\mcA$, we denote the \emph{unknown} probability distribution of $\bX$ by $\P_\ba$. Finally, we define function vector $\bef \triangleq (f_1,\cdots, f_N)^\top$ and function space $\mcF\triangleq \prod_{i=1}^N \mcF_i$.

\subsection{Problem Statement}
The objective in causal bandits is designing a sequence of intervention vectors to optimize a reward measure. Following the convention of the causal bandit literature, we designate $N$ as the reward node. Hence, the expected reward under intervention $\ba\in\mcA$ is given by
\begin{equation}
    \mu_\ba \triangleq \E_\ba \left[ X_{N} \right] \ ,
\end{equation}
where $\E_\ba$ denotes expectation under measure $\P_\ba$. The optimal intervention $\ba^*\in\mcA$ is defined as the intervention that maximizes the reward, i.e.,
\begin{equation}
\label{eq:Opt_Int}
    \ba^* \triangleq \argmax_{\ba\in\mcA}\mu_\ba \ .
\end{equation}
To identify $\ba^*$, the learner sequentially interacts with the underlying causal system. The sequence of interventions over time is denoted by $\{\ba(t)\in\mcA: t\in\N\}$. Upon intervention $\ba(t)$ at time $t\in\N$, the learner observes the graph data $\bX(t)=(X_1(t),\dots, X_N(t))^\top$. We denote the filtration generated by the choices of interventions and associated values by
\begin{equation}
    \mcH_t \triangleq \sigma\left(\ba(1),\bX(1),\dots,  \ba(t),\bX(t), \ba(t+1)\right)\ .
\end{equation}
The learner's objective is to minimize the cumulative regret incurred over time with respect to the best intervention $\ba^*$ in hindsight. We define $r(t) \triangleq \mu_{\ba^*} -\mu_{\ba(t)} $ as the average regret incurred at time $t$ and denote the associated average cumulative regret over horizon $T$ by $R(T)=\sum_{t=1}^T r(t)$, We consider the following two standard cumulative regret measures. 

{\textbf{Frequentist regret:}} We assess the expected cumulative regret denoted by 
    \begin{equation}\label{eq:frequentist_regret}
    \E[R(T)] \triangleq T\mu_{\ba^*} -\sum_{t=1}^{T} \mu_{\ba(t)} \ . 
\end{equation}
{\textbf{Bayesian regret:}} We evaluate the Bayesian regret as the average cumulative regret over the entire class $\mcF$. For this purpose, we denote the cumulative regret associated with $\bef\in\mcF$ by $R_\bef(T)$ and accordingly define 
    \begin{equation}
       {\rm BR}(T) \triangleq \E_{\mcF} \E[R_\bef(T)] \ .
    \end{equation}
We adopt the following regularity assumptions.
\begin{assumption}[Sub-Gaussian noise]\label{ass:noise}
     We assume the noise terms are conditionally sub-Gaussian, i.e., for all $\lambda \in \R$ we have
     \begin{equation}
         \E \left[ \exp(\lambda \epsilon_i(t))\mid \mcH_{t-1} \right] \leq \exp(\lambda^2/2 )\ , \;\;  \lambda\in\R \ .
    \end{equation}
\end{assumption}
\begin{assumption}[Lipschitz continuity]\label{ass:lipschitz} The functions $f_i$ in class $\mcF_i$ are Lipschitz-continuous with Lipschitz constant $K_i$, i.e., for all $a_i\in\mcA_i$ and $x,y\in \R^{d_i}$, we have 
 \begin{equation}
     \left|f_i(x\; ; \;a_i)-f_i(y\; ; \;a_i)\right|\leq K_i\norm{x-y} \ .
 \end{equation}
 \end{assumption}
\begin{assumption}[Function boundedness]\label{ass:bound} For each class $\mcF_i$ there exists a constant $C_i\in\R$ such that 
\begin{equation}
    |f_i(X_{\Pa(i)}\; ; \; a_i)| \leq C_i \ .
\end{equation}
\end{assumption}
This assumption is natural for some SCMs, e.g., neural network SCMs with sigmoid activation functions. In some other SCMs, this ensures the system's boundedness, a common assumption in causal bandits~\citep[e.g.,][]{varici2022causal,sussex2022model}.

\section{PRELIMINARIES} 
We provide some definitions that are instrumental to the analysis and characterizing the regret.

\subsection{Eluder Dimension and Covering Number}
Our regret analyses show that the achievable regret is a function of the inherent complexity of the function classes $\{\mcF_i :i\in[N]\}$. In this section, we specify two measures of functional complexity that are instrumental in specifying achievable regrets. These measures aim to capture an approximate structure of the function class based on target approximation tolerances. The first measure is the \emph{eluder dimension} of a function class \citep{russo2013eluder}. It measures the degree of dependence among the function class as the worst-case sample complexity required to predict the values of unobserved interventions using the observed samples. We refer to \citet{huang2021short} and \citet{li2022understanding} for more discussion on the eluder dimension.
The second measure is the \emph{covering number}, which provides insight into the function class's propensity towards statistical over-fitting. This complexity measure is widely used in analyzing infinite function spaces~\citep{wainwright2019high}. 

To formalize these definitions, we use an extended definition of $\epsilon$-dependence. This definition is instrumental to characterizing how distinguishable the functions $\bef\in\mcF$ are, which can be used to specify a measure of complexity for $\mcF$. Corresponding to each node~$i$, consider $n$ arbitrary inputs to function $f_i$ as specified in~\eqref{eq:SCM}. Denote these inputs by $\{Z_i(m):m\in[n]\}$ where $Z_i (m)\triangleq (X_{\Pa(i)}(m), a_i(m))$. Denote the space of these inputs by  $\mcZ_i\triangleq \R^{d_i}\times \mcA_i$. Based on these, we define $\epsilon$-dependence among the members of $\mcZ_i$ as follows.
\begin{definition}[$\epsilon$-dependent]\label{def:dependence}
    We say $Z_i\in\mcZ_i$ is $\epsilon$-dependent on $\{Z_i(m):m\in[n]\}$ with respect to $\mcF_i$ if any pair of functions $f_i, \tilde f_i\in\mcF_i$ that satisfy
    \begin{align}
        \sum_{m=1}^n(f_i(Z_i(m))-\tilde{f_i}(Z_i(m)))^2\leq \epsilon^2\ ,        
    \end{align}
     also satisfies $|f(Z_i)-\tilde{f}(Z_i)|\leq \epsilon$. 
\end{definition}
Based on this similarity measure, the $\epsilon$-eluder dimension of the function class $\mcF_i$ is defined as follows.
\begin{definition}[$\epsilon$-eluder dimension] 
\label{def:eluder}
The $\epsilon$-eluder dimension ${\rm dim}(\mcF_i, \epsilon)$ is the largest $n$ such that for some $\epsilon'\geq \epsilon$, there exists an ordered set $\{Z_i(m):m\in[n]\}$ where each element is $\epsilon'$-independent of its predecessors. Based on the individual eluder dimensions of the classes, we define the maximum eluder dimension as
\begin{align}
        \operatorname{dim}(\mcF,\epsilon) & \triangleq \max_{i\in[N]} \operatorname{dim}(\mcF_i, \epsilon)\  .
    \end{align}
\end{definition}

\begin{definition}[Covering number] \label{def:covering} For a given function class $\mcF_i$ and any constant $\alpha>0$, we say that the set $\mcC\subseteq \mcF_i$ is an $\alpha$-cover of $\mcF_i$ if for any function $f_i\in\mcF_i$ we can always find another function $\tilde f_i\in\mcC$ such that $\|f_i-\tilde f_i\|_\infty \leq \alpha$.
The $\alpha$-covering number associated with $\mcF_i$ and $\alpha$, denoted by $\cn_\alpha(\mcF_i)$, is the smallest cardinality $\mcC$ among all possible $\alpha$-cover sets~$\mcC$. Accordingly, for the class $\mcF$ we define
\begin{align}
        \cn_{\alpha}(\mcF)& \triangleq  \max_{i\in[N]} \cn_{\alpha}(\mcF_i)\ .
    \end{align}
\end{definition}

\section{GCB ALGORITHMS}
\label{sec:algorithm}
This section serves a two-fold purpose. First, it provides constructive proofs for regret guarantee results for any arbitrary family of functions $\mcF$.
Secondly, it provides algorithms for parametric families of functions that achieve the regret guarantees. More specifically, some decision rules involve solving \emph{functional} optimization problems for choosing an optimal choice of $f_i\in\mcF_i$. Solving such problems has different natures in the parametric (finite-dimensional) and non-parametric (infinite-dimensional) function spaces $\mcF$. 

In the infinite-dimensional setting, it requires functional optimization. Unless solvable via the calculus of variations, solving such problems is difficult to carry out in practice. Nevertheless, a guaranteed existence of a solution suffices for establishing the regret guarantees for the class. On the other hand, in the parametric class of functions, which are the classes encountered in all practical settings, finding an optimal $f_i$ becomes equivalent to solving a parametric optimization problem, which often can be solved numerically, rendering the parameterized family of distributions amenable to practical implementations. For such cases, in addition to the regret guarantees, we also have algorithms that achieve the regret upper bounds.
We present algorithms for the frequentist and Bayesian settings and provide the associated regret analysis in Section~\ref{sec:regret}. 

\noindent\textbf{Algorithm overview.} The central part of our algorithm for the frequentist setting is designing a procedure for estimating the optimal intervention vector $\ba^*$ specified in~\eqref{eq:Opt_Int}. The estimator is designed such that we can progressively, over time, control the size of the confidence set we form for vector function $\bef$ Given the confidence intervals, we adopt a UCB-based strategy to choose interventions $\ba^*$. We refer to this algorithm as the {\bf G}eneral {\bf C}ausal {\bf B}andit UCB (GCB-UCB). For the Bayesian setting, we find the posterior likelihood of the function vector $\bef\in\mcF$ based on the available data and adopt an approach based on Thompson Sampling (TS) to identify an intervention that maximizes the expected posterior reward. We refer to this algorithm as the {\bf G}eneral {\bf C}ausal {\bf B}andit TS (GCB-TS).

\begin{algorithm}[htb]
\caption{GCB-UCB}
\label{alg:eluder_algorithm}
\begin{algorithmic}[1]
\State \textbf{Input:} Horizon $T$, causal graph $\mcG$, intervention set $\mcA$ and function spaces $\mcF$.
\State \textbf{Initialization:} Set $\beta_0(\mcF_i,\delta,\alpha)$ as in \eqref{equ:beta}. 
\For {$t = 1,2,\ldots,T$}
    \State Compute the least squares estimates $\tilde{f}_{i,t}$ as in \eqref{equ:LS_estimates} for $i \in [N]$.
    \State Construct the confidence set $\mcC_{i,t}$ as in \eqref{equ:confidence_set} for $i \in [N]$.
    \State Select interventions $\ba(t)$ according to \eqref{equ:arm_selection}.
    \State Pull $\ba(t)$, observe $\bX(t) = [X_1(t),\dots, X_N(t)]$. \label{line:play_best_arm}
\EndFor
\end{algorithmic}
\end{algorithm}

\begin{algorithm}[htb]
\caption{GCB-TS}
\label{alg:eluder_algorithm_ts}
\begin{algorithmic}[1]
\State \textbf{Input:} Horizon $T$, causal graph $\mcG$, intervention set $\mcA$, function spaces $\mcF$ and prior distribution $\pi_0$.
\For {$t = 1,2,\ldots,T$}
    \State $\tilde \bef \sim \pi_{t-1}(\cdot \mid \mcZ(t)).$
    \State Select interventions $\ba(t)$ according to \eqref{equ:arm_selection_ts}.
    \State Pull $\ba(t)$, observe $\bX(t) = [X_1(t),\dots, X_N(t)]$. \label{line:play_best_arm_ts}
\EndFor
\end{algorithmic}
\end{algorithm}

To formalize the algorithms, by recalling the definition $Z_i (t)\triangleq (X_{\Pa(i)}(t), a_i(t))$, we define 
\begin{equation}
    \mcZ_i(t)\triangleq \{Z_i (s)\;:\; s\in[t-1]\}\ ,
\end{equation}
which is the set of $t-1$ initial interventions on node $i$ and the data generated by its parents. Similarly, we denote the set of all interventions and data generated in the system up to time $t-1$ by 
\begin{equation}
    \mcZ(t) \triangleq \{(\bX(s),\ba(s))\;:\; s\in[t-1] \}\ .
\end{equation}

\subsection{GCB-UCB Algorithm}
We adopt the ordinary least-squares method to estimate the functions. Based on the choices of interventions and the interventional data up to time $t$ collected in $\mcZ_i(t)$, we form an estimate for $f_i$ at time $t$ as follows.
\begin{equation}
        \label{equ:LS_estimates}
    \tilde{f}_{i,t} \triangleq  \argmin_{f_i\in \mcF_i}\sum_{s=1}^{t-1}\Big[f_i(X_{\Pa(i)}(s) \, ; \, a_i(s))-  \!X_i(s)\Big]^2 .
\end{equation}
Based on this estimate, we construct a confidence interval for the estimated function $\tilde{f}_{i,t}$ as follows.
\begin{equation}
    \label{equ:confidence_set}
    \!\! \mcC_{i,t} \! \triangleq \! \left\{ g \in\mcF_i:\norm{g-\tilde{f}_{i,t}}_{\mcZ_i(t)}\!\! \leq \! \sqrt{\beta_t(\mcF_i,\delta,\alpha_i)} \right\} ,
\end{equation}
in which $\beta_t(\mcF_i,\delta,\alpha_i)$ controls the radius of the confidence interval, and it is specified based on the $\alpha_i$-covering number of the class $\mcF_i$, i.e., $\cn_{\alpha_i}(\mcF_i)$, as follows. 
\begin{align}
     &\beta_t(\mcF_i,\delta,\alpha_i) \triangleq \nonumber  \\
   \label{equ:beta}  & \ \ \  8 \log\left(\frac{\cn_{\alpha_i}(\mcF_i)}{\delta}\right) + 2\alpha_i t\left( \! 8C_i+\sqrt{8\ln\frac{4t^2}{\delta}}\right) .
\end{align}
Besides the covering number, an intrinsic property of the function space, the radius $\beta_t(\mcF_i,\delta,\alpha_i)$ depends on $\alpha_i$ and $\delta$. The role of $\alpha_i$ is controlling the discretization parameter for covering number, which also influences the second summand that is related to the uncertainty of estimators. Furthermore, the role of $\delta$ is controlling the confidence level of the confidence set. Based on this, the confidence radius has two components. One component is a constant (time-invariant) that depends on the covering number $\cn_{\alpha}(\mcF_i)$ and is decreasing with $\delta$. The second component is increasing in time $t$ and is decreasing with $\delta$. Finally, we define the confidence space for the entire causal system as $\mcC_{t}\triangleq \prod_{i=1}^{N} \mcC_{i,t}$. Based on these definitions, corresponding to intervention $\ba\in\mcA$ at time $t$ we define the upper confidence bound as follows.
\begin{equation}
    \operatorname{UCB}_{\ba}(t) \triangleq \max_{\bef\in\mcC_{t}} \E_\ba[X_N\mid \bef ]\ .
\end{equation}
Accordingly, the intervention at time $t$ is specified by
\begin{equation}
\label{equ:arm_selection}
    \ba(t) \triangleq \argmax_{\ba\in \mcA} \operatorname{UCB}_\ba(t) \ ,
\end{equation}
and finally, we define
\begin{align}
\label{eq:barf}
     \bar \bef_t \triangleq \argmax_{\bef\in\mcC_{t}} \operatorname{UCB}_{\ba(t)}(t)\ .
\end{align}

\subsection{GCB-TS Algorithm}
A TS-based algorithm gradually refines the posterior distributions of the reward for each intervention and selects the interventions by sampling from their posterior distributions. The interventions are selected sequentially in a way that they balance the exploitation and the exploration processes. To this end, a TS algorithm must update all arms' posterior distributions. For the GCB-TS algorithm, we assume a prior $\pi_0$ is known in advance, and it is periodically updated over time to construct the posterior $\pi_t$ after gathering the observed samples $\mcZ(t)$. The GC-TS algorithm at time $t$, samples a function $\tilde{\bef}$ from the posterior $\pi_{t-1}$. Subsequently, an intervention is selected such that it maximizes the expected reward under $\tilde{\bef}$, i.e., 
\begin{equation}
\label{equ:arm_selection_ts}
    \ba(t) = \argmax_{\ba\in \mcA} \E_\ba[X_N\mid \tilde \bef ]\ .
\end{equation}

\section{REGRET GUARANTEE}
\label{sec:regret}
This section has two parts. In Section~\ref{sec:regret:bounds}, we establish regret upper and lower bounds for a general class of functions $\mcF$. These bounds explicitly capture the impact of the causal graph structure (through $d$ and $L$) and the function class $\mcF$ (via its eluder dimension, covering number, and Lipschitz constants). In Section~\ref{sec:refinedbound}, we specialize and refine these bounds for three special classes of SCMs: linear, polynomial, and neural network, which cover the prevalent SCMs.

\subsection{Regret Upper Bounds}
\label{sec:regret:bounds}

We first provide upper bounds that capture the performance of the GCB-UCB and GCB-TS algorithms, followed by a minimax lower bound that is independent of the function class $\mcF$. The key step in these analyses is evaluating the impact of the causal structure on the regret. Specifically, once we form an estimate of the function $f_i$ for any node $i\in[N]$, the estimation error for this node propagates through the causal path from node $i$ to node $N$, affecting the estimation errors of all the nodes on the path as well as the final reward node. We refer readers to this as the \emph{compounding} effect of the causal graph.

We begin with recalling a concentration lemma provided in \citet{russo2013eluder}. This lemma addresses the confidence set's precision. Specifically, this lemma shows that the confidence set $\mcC_{i,t}$ contains $f_i$  with a high probability at all times and specifies the precision of the estimates $\{\tilde f_{i,t}: t\in\N\}$ formed by~\eqref{equ:LS_estimates}. 
\begin{lemma}
\label{lemma:concentration} 
For any $i\in[N]$ and $\forall \delta>0$ and $\forall \alpha>0$
\begin{equation}
    \P(f_i\in \mcC_{i,t}) \geq 1-2\delta\ , \qquad \forall t\in[T]\ .
\end{equation}
\end{lemma}

Next, we specify the discretization parameter $\alpha_i$ based on the function space $\mcF_i$. When the function space is finite, we choose $\alpha_i$ to be the minimum distance with respect to the infinity norm between any two functions in $\mcF_i$. 
If the function space is continuous, however, the minimum distance may be zero. In the latter case, we truncate $\alpha_i$ at $\frac{1}{T}$, resulting in the following specification overall:
\begin{equation}
\label{equ:alpha_i}
    \alpha_i \triangleq  \max\bigg\{\frac{1}{T},\inf_{f_1\neq f_2\in\mcF_i}\norm{f_1-f_2}_{\infty}\bigg\}\ .
\end{equation}
The choice of $\frac{1}{T}$ can be understood by considering the two opposing effects of $\alpha_i$ in the two summands of $\beta_t(\mcF_i,\delta,\alpha_i)$, the confidence radii defined in~\eqref{equ:beta}. On the one hand, increasing $\alpha_i$ decreases the covering number $\cn_{\alpha_i}(\mcF_i)$, which in turn decreases the first summand. On the other hand, increasing $\alpha_i$ increases the second summand in $\beta_t(\mcF_i,\delta,\alpha_i)$. Hence, selecting $\alpha_i = \frac{1}{T}$ balances the two summands. 
Throughout the rest of the paper, we use the shorthand $\operatorname{dim}\left(\mcF\right)$ for $\operatorname{dim}\left(\mcF,\frac{1}{T}\right)$ and the shorthand $\cn(\mcF)$ for $\cn_{\frac{1}{T}}(\mcF)$.

For characterizing the \emph{compounding} effect of the causal graph along various causal paths, we define the following derived Lipschitz constants. 
First recall that we have defined $L_i$ as the causal depth of node $i$, which is the length of the longest causal path that starts from the root nodes of the graph and ends at node $i$. Accordingly, we define $ K^{(\ell)}$ as the maximum Lipschitz constant among the constants of the  nodes that have causal depth $\ell\in[L]$, i.e., 
\begin{equation}
    K^{(\ell)} = \max_{i\in[N], L_i=\ell} K_i \ .
\end{equation}
Finally, we define $K=\prod_{\ell=2}^{L} K^{(\ell)}$.

We present the following lemma to bound the cumulative error at a node $i\in[N]$ due to the compounding of estimator uncertainty along causal paths. This is a key step in characterizing the regret at the reward node because the latter 
hinges on the compounded estimation errors at all intermediate nodes along the 
causal paths to the reward node. 
In our analysis, we find that the contribution of node $i$ at time $t$ is determined by the difference in the expected value of $X_i$ under the optimistic function $\bar \bef_t$, which maximizes $\operatorname{UCB}$ over the confidence set, and the true function $\bef$. The next lemma bounds these differences. 
\begin{lemma}[Compounding Error]\label{lm:bound_l_paths}
If $f_i\in \mcC_{i,t}$ for all $i\in[N]$ and $t\in[T]$, then the following error bound holds
\begin{align}
\label{equ:cumulative bound}
    \sum_{t=1}^{T}&\Big|\E_{\ba(t)}[X_{i}\mid \bar \bef_t] - \E_{\ba(t)}[X_{i}\mid \bef]\Big| \nonumber\\
    &\qquad \leq \mcB(\mcF,\delta) \sum_{\ell=2}^{L_i}  d^{\ell-1} \prod_{k=2}^{\ell} K^{(k)}\ .
\end{align}
\end{lemma}
The quantity $\mcB(\mcF,\delta)$ appearing in Lemma~\ref{lm:bound_l_paths} is defined as follows, where for each node $i\in[N]$, we define $\mcB_i(\mcF_i,\delta)$ to represent the maximum accumulated error when $L_1=1$ for $i\in[N]$.
 \begin{align}
        \mcB_i(\mcF_i,\delta) & \triangleq  1  + \min\{\operatorname{dim}(\mcF_i,\alpha_i),T\} C_i \nonumber\\
        & \quad + 4\sqrt{\operatorname{dim}(\mcF_i,\alpha_i)\beta_T(\mcF_i,\alpha_i,\delta) T} \ ,\\
        \label{eq:beta}\mcB(\mcF,\delta) &\triangleq  \max_{i\in[N]} \mcB_i(\mcF_i,\delta)\ .
    \end{align}      
When we have eluder dimension $\operatorname{dim}(\mcF_i,\alpha_i) = o(\sqrt{T})$, i.e., when $\mcB_i(\mcF_i,\delta)=o(T)$, for all nodes $i\in[N]$, $\mcB(\mcF,\delta)$ defined in~\eqref{eq:beta} scales as 
\begin{equation} 
\mcO\left(\sqrt{T\; \operatorname{dim}\left(\mcF\right)\beta_T\left(\mcF,\frac{1}{T},\delta\right) }\right)\ .
\end{equation}    

Leveraging the result of Lemma~\ref{lm:bound_l_paths}, we characterize the achievable regrets of GCB-UCB and GCB-TS. 
\begin{theorem}[Regret Upper Bound]\label{thm:regret}
    Under Assumptions~\ref{ass:noise}-\ref{ass:bound}, by setting $\delta=\frac{1}{NT}$, for the GCB-UCB algorithm, we have
    \begin{align}
        &\E[R(T)] \nonumber\\
        & \leq 4C_N + \mcB\left(\mcF,\frac{1}{NT}\right) \sum_{\ell=0}^{L} d^{\ell-1} \prod_{j=2}^{\ell} K^{(j-1)}\\
         &= \mcO\left(K d^{L-1}  \sqrt{T\; \operatorname{dim}(\mcF) \log \big( NT \cn(\mcF) \big) }\right) \ .
    \end{align}
\end{theorem}
 
\begin{corollary}[Bayesian Regret Upper Bound]
\label{cor:regret}
Under Assumptions~\ref{ass:noise}-\ref{ass:bound}, by setting $\delta=\frac{1}{NT}$,
for the GCB-TS algorithm, we have
      \begin{align}
        &{\rm BR}(T) = \mcO \left(K d^{L-1}  \sqrt{T\; \operatorname{dim}\left(\mcF\right) \log \big( NT \cn(\mcF) \big) }\right) \ .
    \end{align}
\end{corollary}
Based on the results established in Theorem~\ref{thm:regret}, we have the following key observations.

\vspace{-.05 in}
\begin{itemize}
        \item \textbf{Graph size.} One important observation is that the regret bounds depend only logarithmically on graph size $N$. More importantly, we will show that for all the widely-used special cases, the dependence on $N$ behaves as $\sqrt{\frac{\log N}{\log T}}$, indicating that the graph size has a diminishing effect in $T$. To highlight the significance of this, note that in the state-of-the-art results \citep{varici2022causal,sussex2022model}, the achievable regret bounds for less general SCMs (linear and Gaussian processes) scale with $\sqrt{N}$ and $N$, respectively.  Furthermore, the regret bounds depend on the graph through its connectivity structure specified by $d$ and $L$.
        
         \item  \textbf{Impact of function classes.} The upper bound depends linearly on the square root of the eluder dimension and logarithm of the covering number, reflecting the influence of the function space's complexity on the regret's growth. Furthermore, regret bounds depend on the Lipschitz coefficient $K$, which is expected.

    \item \textbf{Graph parameters.}
    The regret upper bound has polynomial dependence on maximum degree $d$ and exponential dependence on maximum causal depth $L$.
    
    \item \textbf{Intervention space.} Another observation is the lack of dependence on the size of the intervention space.  This allows for adopting intervention spaces with any desired level of granularity. 
\end{itemize}

\subsection{Refined Regret Bounds for Special SCMs}\label{sec:refinedbound}

In this section, we discuss three special classes of SCMs and provide refined upper and corresponding lower bounds on regret for them. These results are summarized in Table~\ref{tab:bounds1}. The proofs and additional details are presented in Appendices~\ref{sec:refinedupperbounds} and \ref{sec:prooflower}. 

For the lower bound, we consider the minimax lower bound on the regret when considering the specific function classes $\mcF$.
We define $\bar\mcG$ as the set of all causal graphs specified over $N$ nodes with parameters $d$ and $L$, and $\Pi$ as the set of all possible bandit policies. The minimax regret is defined as 
\begin{equation}
\label{eq:minimax}
    \inf_{\pi\in\Pi} \sup_{\mcF\in\mcF}\sup_{\mcG\in\bar\mcG} \E[R(T)]\ .
\end{equation}

\noindent\textbf{Linear SCMs.} For given vectors $\theta_{i}, \bar\theta_{i}\in \R^{d_i}$ and scalar $a_i\in\mcA_i$, set
    \begin{equation}
        f_i(X_{\Pa(i)}\; ;\; a_i)=\langle\theta_{i}(1-a_i) + a_i \bar \theta_{i},X_{\Pa(i)} \rangle\ ,
    \end{equation}
    and accordingly, define
    \begin{equation}
    \label{eq:SCM_lin}
    \mcF_i = \{f_i :\theta_{i}, \bar\theta_{i}\in \R^{d_i}, \norm{\theta_i}\leq K_i, \norm{\bar \theta_i}\leq K_i \} \ ,
\end{equation}
where the norm constraints are enforced to satisfy Assumptions~\ref{ass:lipschitz}-\ref{ass:bound}. 
By setting $\mcA_i=\{0,1\}$, the generalized soft intervention model reduces to the linear SCMs with binary atomic interventions. We note that Assumption~\ref{ass:bound} implies that the noise terms are bounded under linear SCMs, which is similar to the assumption adopted in \cite{varici2022causal}.

\begin{theorem} \label{th:SCM_lin}
    For the linear class of SCMs specified in~\eqref{eq:SCM_lin} and $\mcA_i\in\{0,1\}$ we have
  \begin{align*}
\operatorname{dim}(\mcF_i) = \tilde \mcO (d \log T)\ , \;\; \mbox{and} \;\;    \log \cn(\mcF_i) = \tilde \mcO (d \log T)\ , 
\end{align*}
and the upper bound in Theorem~\ref{thm:regret} becomes 
\begin{align*}
    \mcO(Kd^{L}\sqrt{T\log T\varpi_{\rm L}})\ , \; \; \varpi_{\rm L}=\log T \left[1+\frac{\log N}{d\log T}\right] \ .
\end{align*}
\end{theorem}
Theorem~\ref{th:SCM_lin} shows that achievable regret has a diminishing dependence on $N$ (as $T$ increases). This is a substantial improvement over the known regret result, i.e., $\mcO(K d^{L-\frac{1}{2}}\sqrt{NT} \log T)$, which grows with $\sqrt{N}$.

\begin{theorem}[Minimax Lower Bound for Linear SCMs]\label{thm:lowerlinear}
Under linear SCMs defined in \eqref{eq:SCM_lin}, given graph parameters $d$ and $L$ and Lipschitz constant $K$, there exists a causal bandit instance  for which the expected regret of any algorithm is at least
\begin{equation}
    \E[R(T)] \geq \Omega\left(K d^{\frac{L}{2}-1}\sqrt{T}\right) \ .
\end{equation}
\end{theorem}

\noindent\textbf{Polynomial SCMs.} 
For a given vector $\theta_{i}\in \R^{d_i+1}$ and a polynomial maximum degree $p\in\N$, set
\begin{align}
\label{eq:functionpoly}
        f_i(X_{\Pa(i)}\; ;\; a_i)=\left(\langle\theta_{i}\; , \;  [X_{\Pa(i)} , a_i] \rangle\right)^p \ ,
\end{align}
in which the intervention parameter $|a_i|\leq 1$ controls the coefficients in the polynomial. Note that the polynomial can be defined in different forms, and the one selected is only a representative. Accordingly, define 
\begin{align}\label{eq:SCM_poly}
    \mcF_i  =  \left\{f_i: \theta_i \in \R^{d_i+1}, \norm{\theta_i}\leq \frac{K_i^{1/p}}{p^{1/p}(d C+1)} \right\}\ ,
\end{align}
where $C=\max_{i\in[N]} C_i$ and the norm constraints are enforced to satisfy Assumptions~\ref{ass:lipschitz}-\ref{ass:bound}. Similarly, Assumption~\ref{ass:bound} implies the noises are bounded in this case.
\vspace{0.01 in}
\begin{theorem} \label{th:SCM_poly}
    For the polynomial class of SCMs specified in~\eqref{eq:SCM_poly} for $p=2$ we have
 \begin{align}
\operatorname{dim}(\mcF_i) & = \tilde \mcO \big( (d+1)^2 \log T\big)\ ,\\
\mbox{and} \qquad  \log \cn(\mcF_i) & = \tilde \mcO ((d+1)^2 \log T)\ ,
\end{align}
and the upper bound in Theorem~\ref{thm:regret} becomes 
{\small 
\begin{align*}
    \mcO(Kd^{L+1}\sqrt{T\log T\varpi_{\rm P}})\ , \;\;  \varpi_{\rm P}=\log T \left[1+\frac{\log N}{d\log T}\right] \ .
\end{align*}
}
\end{theorem}

\begin{theorem}[Minimax Lower Bound for Polynomial SCMs]\label{thm:lowerpoly}
Under polynomial SCMs defined in \eqref{eq:SCM_poly}, given graph parameters $d$ and $L$ and Lipschitz constant $K$, there exists a causal bandit instance for which the expected regret of any algorithm is at least
\begin{equation}
    \E[R(T)] \geq \Omega\left(K \sqrt{T}\right) \ .
\end{equation}
\end{theorem}

\noindent\textbf{Neural Network SCMs.}
For a given matrix $\bTheta_1\in \R^{s \times (d_i+1)}$ and vector $\bTheta_{2} \in \R^{s}$, set $f_i$ as a two-layer neural network  with the input vector $[a_i,X_{\Pa(i)}]$, i.e., 
\begin{equation}\label{eq:NN}
    f_i(X_{\Pa(i)} \f a_i) = \sigma\big( \bTheta_{2} \cdot  \; \sigma(\bTheta_1  \cdot [a_i,X_{\Pa(i)}])\big)\ ,
\end{equation}
where $\sigma$ is an increasing continuous activation function with bounded gradients. Accordingly, we define $\mcF_i$ as
\begin{equation}
    \label{eq:SCM_nn}
    \mcF_i = \{f_i : \bTheta_1\in \R^{s \times (d_i+1)}, \bTheta_{2} \in \R^{s}\}\ .
\end{equation}

\begin{theorem} \label{th:SCM_nn}
    For the neural network class of SCMs specified in~\eqref{eq:SCM_nn} with width $s\geq d$, the upper bound in Theorem~\ref{thm:regret} becomes 
\begin{align*}
    \mcO\Big(\bar Kd^{L-1}\sqrt{T\log T\; \varpi_{\rm N}}\Big)\ , \; \varpi_{\rm N} = s \log T\left[1+\frac{\log N}{d\log T} \right] ,
\end{align*}
where $\bar K$ is an upper bound of Lipschitz constant $K$.
\end{theorem}

\begin{theorem}[Minimax Lower Bound for Neural Network SCMs]\label{thm:lowernn}
Given graph parameters $d$ and $L$ and Lipschitz constant $K$, there exists a causal bandit instance with neural network SCM defined in \eqref{eq:SCM_nn} for which the expected regret of any algorithm is at least
\begin{equation}
    \E[R(T)] \geq \Omega\left(K d^{\frac{L}{2}-1}\sqrt{T}\right) \ .
\end{equation}
\end{theorem}

We note that in our definition of the class of neural networks in~\eqref{eq:NN}, we have an activation function in the second layer, which is not standard in neural networks. This choice was made to simplify the results presented. Besides, if activation functions are allowed to vary between layers, we can set the activation function in the second layer to be the identity mapping, which will result in a similar result for the widely used definition of neural network.

We note that the difference between lower and upper bounds in linear and neural network SCMs (i.e., $d^{L/2}$ versus $d^L$) has reasons fundamentally beyond our setting. Specifically,  the regret of the linear bandit algorithm called optimism in the face of uncertainty (OFUL)~\citep{abbasi2011improved} is $\mcO(d\sqrt{T})$. In contrast, the lower bound for linear bandits with finite arms is $\mcO(\sqrt{dT})$, yielding a difference of $d^{1/2}$ versus $d$ in the lower and upper bounds. The linear bandit problem can be considered a special case of our framework with linear SCM and $L = 1$. In essence, we have a similar gap compounded over $L$ layers. Besides, we expect the missing graph parameters in the lower bound of quadratic SCMs are due to the large Lipschitz constant in this case (see Appendix~\ref{sec:poly} for details). 

\section{CONCLUSION}
In this paper, we have provided regret analysis for causal bandits with general structural equation models (SCMs) and generalized soft interventions. The achievable regret bounds are shown to depend on inherent complexity measures of the class of functions encompassing the SCMs, namely the eluder dimension and the covering number. Furthermore, the achievable regret depends on the graph topology primarily through its connectivity parameters (maximum degree and the length of the causal paths) and has a diminishing dependence on the size of the graph. Furthermore, we have delineated a universal minimax regret lower bound. Finally, we have customized the results and refined the bounds for three special classes of SCMs: linear, polynomial, and neural network SCMs. 

\section*{Acknowledgments}
This work was supported by the Rensselaer-IBM AI Research Collaboration (AIRC), part of the IBM AI Horizons Network.

\bibliography{reference}
\bibliographystyle{apalike}

\clearpage
\appendix
\onecolumn

\hsize\textwidth
 \linewidth\hsize 
\hrule height4pt
\vskip .25in
{\centering
  {\Large\bfseries Causal Bandits with General Causal Models and Interventions\\ Supplementary Materials \par}}
\vskip .25in
\hrule height1pt
\vskip .25in

\doparttoc 
\faketableofcontents 

\part{} 
\parttoc 

\section{Regret Upper Bound -- General SCM}
\label{sec:proofupper}
\subsection{Proof of Lemma~\ref{lm:bound_l_paths}}
Throughout this proof, we use  $\beta_t$ as a shorthand to refer to $\beta_t(\mcF_i,\delta,\alpha_i)$. We define $\bar\bef_t\triangleq ( \bar f_{1,t},\dots, \bar f_{N,t})$. Based on these, we first prove that for the nodes with causal depth 0, i.e., $L_i=0$, we have
Consider node $i\in[N]$ with causal depth $L_i=0$, in this case we have
\begin{equation}
    [X_{i}(t) \mid \bar{\bef}_t] = [X_{i}(t)\mid \bef] = \epsilon_i(t) \ .
\end{equation}
This holds since the nodes with causal depth 0 have no parent nodes, and their values are entirely controlled by noise. Hence, we obtain
\begin{equation} \sum_{t=1}^{T}\E_{\ba(t)}\left|[X_{i}(t) \mid \bar{\bef}_t] - [X_{i}(t)\mid \bef] \right| =\sum_{t=1}^{T}\E_{\ba(t)}\left| \epsilon_i(t) - \epsilon_i(t)\right| =0 \ .
\end{equation}

In the next step, we show that for all nodes $i\in[N]$ with causal depth at least 1, i.e., $L_i\geq 1$, we have the following upper bound on the cumulative average error terms:
\begin{align}
\label{equ:lemma5aim}
     \sum_{t=1}^{T} \E_{\ba(t)}\left|[X_{i}(t) \mid \bar{\bef}_t] -  [X_{i}(t)\mid \bef] \right| \leq \mcB(\mcF,\delta) \sum_{\ell=1}^{L_i}  d^{\ell-1} \prod_{k=2}^{\ell} K^{(k)}\ ,
\end{align}
in which we adopt the convention of $\prod_{k=2}^{1} K^{(k)}=1$.
We establish this via induction on the causal depth $L_i$.

\textbf{Base Step: $L_i=1$.}
For node $i\in[N]$ with causal depth $L_i=1$, we show that
\begin{equation}
\sum_{t=1}^{T}\E_{\ba(t)}\left| [X_{i}(t)\mid \bar{\bef}_t] - [ X_{i}(t)\mid \bef]\right|  \leq \mcB(\mcF,\delta) \ .
\end{equation}
When the causal path of a node is $L_i=1$, according to SCM defined in \eqref{eq:SCM}, we have the following expansion:
\begin{align}
   \E_{\ba(t)} \left|[X_{i}(t)\mid \bar \bef_t] - [ X_{i}(t)\mid \bef]\right| &= \E_{\ba(t)} \left| \bar f_{i,t}(X_{\Pa(i)}(t) \f a_i(t)) -f_i(X_{\Pa(i)}(t) \f a_i(t))\right| \\
   & =\E_{\ba(t)}  [\Delta_t(X_{\Pa(i)}(t) \f a_i(t))] \ ,
\end{align}
where we have defined the  difference function $ \Delta_{t}$ for input $Z_i(t) \in \mcZ_i$ as
\begin{equation}
\label{equ:delta}
    \Delta_{t}(Z_i(t) ) \triangleq  \left| \bar f_{i,t}(Z_i(t)) -f_i(Z(t))\right|\ \ .
\end{equation}
Next, we provide the following lemma for bounding the sum of difference functions over time. 
\begin{lemma}
\label{lemma:bounddelta}
    For each node $i\in[N]$, and any given sequence $\{Z_i(t):t\in[T]\}$, where we had defined $Z_i (t)\triangleq (X_{\Pa(i)}(t), a_i(t))$  we have
    \begin{equation}
   \sum_{t=1}^{T} \Delta_t(Z_i(t)) \leq \mcB(\mcF,\delta)\ .
    \end{equation}
\end{lemma}
\begin{proof}
We first leverage the following inequality
\begin{equation}
\label{eq:bounddeltabyZ}
    \Delta_t(Z_i(t)) \leq \sup _{f, \bar{f} \in \mcC_{i,t}} \bar{f}(Z_i(t))-f(Z_i(t)) \ .
\end{equation}
Based on this inequality, we show the following inequality, which specifies the number of times that the error exceeds a certain threshold for any realization of noise vectors $\mcE = \{\bepsilon(t), t\in[T]\}$ and actions vectors $\{\ba(t), t\in[T]\}$. Specifically, we show that for the confidence sets $\mcC_{i,t}$ defined in \eqref{equ:confidence_set},  for any inputs sequence $Z_i(t)\in\mcZ_i$, for any horizon $T\in\N$, and any $\epsilon\in\R^+$, we almost surely have
    \begin{equation}\label{eq:VR_ineq}
        \sum_{t=1}^{T} \mathds{1}\left(\Delta_t(Z_i(t))>\epsilon\right)
        \leq \left(\frac{4\beta_T}{\epsilon^2}+1\right) \operatorname{dim}(\mcF_i,\epsilon) \ .
    \end{equation}
This inequality is a result of~ \citep[Proposition 3]{russo2013eluder} and \eqref{eq:bounddeltabyZ}. Next, we reorder the sequence   $\{\Delta_t(Z_i(t))\;:\; t\in[T]\}$ in the descending order denoted by  $\Delta_{i_1}\geq \Delta_{i_2} \geq \cdots \geq \Delta_{i_T}$. Subsequently, for any given $\alpha_i$ we have the following expansion:
\begin{equation}\label{eq:summands}
    \sum_{t=1}^{T} \Delta_t(Z_i(t)) =  \sum_{t=1}^{T} \Delta_{i_t} \mathds{1}\left\{\Delta_{i_t}\leq \alpha_i\right\} + \sum_{t=1}^{T} \Delta_{i_t} \mathds{1}\{\Delta_{i_t} > \alpha_i\} \ .
\end{equation}
Next, we provide upper bounds for the two summands of~\eqref{eq:summands}. For the first one, if $\alpha_i=\frac{1}{T}$, we have
\begin{equation}\label{eq:delta_bound1}
    \sum_{t=1}^{T} \Delta_{i_t} \mathds{1}\left\{\Delta_{i_t} \leq \alpha_i\right\}  \leq \sum_{t=1}^{T} \frac{1}{T} = 1 \ .
\end{equation}
On the other hand, when $\alpha_i>\frac{1}{T}$, it means that the minimum distance of any two functions is larger than $\frac{1}{T}$. This, in turn, indicates that event $\{\Delta_{i_t} \leq \alpha_i\}$ never occurs since the minimum gap between two functions is less than $\alpha_i$. Hence, for  $\alpha_i>\frac{1}{T}$ we have
\begin{equation}\label{eq:delta_bound2}
    \sum_{t=1}^{T} \Delta_{i_t} \mathds{1}\left\{\Delta_{i_t} \leq \alpha_i\right\}  =0\  .
\end{equation}
As a result, from~\eqref{eq:delta_bound1}~and~\eqref{eq:delta_bound2} we have
\begin{equation}\label{eq:delta_bound4}
    \sum_{t=1}^{T} \Delta_{i_t} \mathds{1}\left\{\Delta_{i_t} \leq \alpha_i\right\}  \leq  1 \ .
\end{equation}
Next, we bound the second summand in~\eqref{eq:summands}. Based on Assumption~\ref{ass:bound} we know that  
\begin{equation}
    \label{eq:bounddeltaC}
    \Delta_{i_t}\leq C_i \ .
\end{equation}
Furthermore, note that by the definition of  $\Delta_{i_t}$  for any $\gamma>0$ we have
\begin{align}\label{eq:delta_bound3}
    \sum_{t=1}^{T} \mathds{1}\{\Delta_{i_t}> \Delta_{i_s} -\gamma\} \geq s\ , \qquad s\in[T]\ .
\end{align}
Combining the observations in~\eqref{eq:VR_ineq}~and~\eqref{eq:delta_bound3} we have
\begin{equation}
    s\leq \left(\frac{4\beta_T}{(\Delta_{i_s} -\gamma)^2}+1\right) \operatorname{dim} (\mcF_i, \Delta_{i_s} -\gamma)\ , \qquad s\in[T]\ .
\end{equation}
Next, note that the eluder dimension is a non-increasing function of the discretization parameter. Thus,  when $\Delta_{i_s}- \gamma \geq \alpha_i$ we have $\operatorname{dim} (\mcF_i, \Delta_{i_s} -\gamma) \leq \operatorname{dim} (\mcF_i, \alpha_i)$, resulting in
\begin{equation}
    s\leq \left(\frac{4\beta_T}{(\Delta_{i_s} -\gamma)^2}+1\right) \operatorname{dim} (\mcF_i, \alpha_i) \ , \qquad s\in[T]\ .
\end{equation}
By rearranging the terms in the above inequality, when $s>\operatorname{dim}(\mcF_i,\alpha_i)$, we find that for any $\gamma>0$ have 
\begin{equation}
    \label{eq:bounddelta}
    \Delta_{i_s} - \gamma < \sqrt{\frac{4\beta_T\operatorname{dim}(\mcF_i,\alpha_i)}{s-\operatorname{dim}(\mcF_i,\alpha_i)}}\ ,  \qquad s\in[T]\ .
\end{equation}
By taking supremum of the left-hand side over $\gamma$ in \eqref{eq:bounddelta}, we conclude that when $\Delta_{i_s}>\alpha_i$ and $s>\operatorname{dim}(\mcF_i,\alpha_i)$ we have 
\begin{equation}
 \label{eq:bounddeltafinal}
        \Delta_{i_s} \leq  \sqrt{\frac{4\beta_T\operatorname{dim}(\mcF_i,\alpha_i)}{s-\operatorname{dim}(\mcF_i,\alpha_i)}} \ , \qquad s\in[T]\ .
\end{equation}

Hence, we can expand and upper bound the second summand of \eqref{eq:summands} as follows.
\begin{align}
    \sum_{t=1}^{T} \Delta_{i_{t}} \mathds{1}\{\Delta_{i_t} > \alpha_i\} &= \sum_{t=1}^{\operatorname{dim}(\mcF_i,\alpha_i)} \Delta_{i_{t}} \mathds{1}\{\Delta_{i_t} > \alpha_i\}+\sum_{t=\operatorname{dim}(\mcF_i,\alpha_i)+1}^{T} \Delta_{i_{t}} \mathds{1}\{\Delta_{i_t} > \alpha_i\}\\    
    & \overset{\eqref{eq:bounddeltaC},\eqref{eq:bounddeltafinal}}{\leq} \min\{\operatorname{dim}(\mcF_i,\alpha_i)\;,\; T\}C_i + \sum_{t=\operatorname{dim}(\mcF_i,\alpha_i)+1}^{T} \sqrt{\frac{4 \beta_T \operatorname{dim}(\mcF_i,\alpha_i)}{t-\operatorname{dim}(\mcF_i,\alpha_i)}} \\
    & \leq \min\{\operatorname{dim}(\mcF_i,\alpha_i)\;,\; T\}C_i + 2\sqrt{\beta_T\operatorname{dim}(\mcF_i,\alpha_i)}\sum_{t=\operatorname{dim}(\mcF_i,\alpha_i)+1}^{T} \frac{1}{\sqrt{t} }\\
     &\leq \min\{\operatorname{dim}(\mcF_i,\alpha_i)\;,\; T\}C_i + 2\sqrt{\beta_T\operatorname{dim}(\mcF_i,\alpha_i)}\int_{t=0}^T \frac{1}{\sqrt{t}} {\rm d} t\\
   \label{eq:delta_bound5}  &\leq \min\{\operatorname{dim}(\mcF_i,\alpha_i)\;,\; T\}C_i + 4\sqrt{\operatorname{dim}(\mcF_i,\alpha_i)\beta_T T} \ .
\end{align}
Thus, in conclusion, based on~\eqref{eq:summands}, \eqref{eq:delta_bound4}, and \eqref{eq:delta_bound5} we have 
\begin{align}
    \sum_{t=1}^{T} \Delta_t(Z_i(t))  
    &\leq 1 + \min\{\operatorname{dim}(\mcF,\alpha_i),T\}C_i + 4\sqrt{\operatorname{dim}(\mcF,\alpha_i)\beta_T T}\\
    &\leq \mcB(\mcF,\delta)\ .
\end{align}
\end{proof}
Subsequently,
\begin{align}
     \sum_{t=1}^{T} \E_{\ba(t)}\left|[X_{i}(t) \mid \bar{\bef}_t] -  [X_{i}(t)\mid \bef] \right|& =  \sum_{t=1}^{T} \E_{\ba(t)}  [\Delta_t(Z(t))] \leq \mcB(\mcF,\delta)\ ,
\end{align}
which proves the result for $L_i=1$.

\textbf{Induction Step.} Assume that the property holds true for causal depths up to $L_i=m$. We show that it will also hold for $L_i=m+1$. For this purpose, we start with the following expansion and apply the triangular inequality to find an upper bound for it.
\begin{align}
\sum_{t=1}^{T}\E_{\ba(t)} & \left|[X_{i}(t)\mid \bar \bef_t] - [X_{i}(t) \mid \bef] \right| \\
&= \sum_{t=1}^{T}  \E_{\ba(t)} \left| \bar  f_{i,t}\left( [X_{\Pa(i)}(t)\mid \bar \bef_t] \f a_i(t)\right) - f_i\left( [X_{\Pa(i)}(t)\mid \bef] \f a_i(t) \right) \right|\\
    \nonumber &\leq \sum_{t=1}^{T} \E_{\ba(t)}  \left| \bar f_{i,t}\left( [X_{\Pa(i)} (t) \mid \bar \bef_t]  \f a_i(t) \right) - \bar  f_{i,t}\left( [X_{\Pa(i)}(t)\mid \bef] \f a_i(t) \right) \right|\\
    \label{equ:mid_of_lemma1} & \hspace{1 in} + \sum_{t=1}^{T} \E_{\ba(t)} \left| \bar  f_{i,t}\left( [X_{\Pa(i)}(t) \mid \bef] \f a_i(t) \right) - f_i\left( [X_{\Pa(i)}(t)\mid\bef] \f a_i(t) \right) \right|\\
    &\leq K_i \sum_{t=1}^{T}  \E_{\ba(t)} \norm{[X_{\Pa(i)}(t) \mid  \bar \bef_t] - [X_{\Pa(i)}(t)\mid \bef]} + \mcB(\mcF,\delta) \ \label{equ:mid_of_lemma} \\
    & \leq K_i \sum_{j\in\Pa(i)} \sum_{t=1}^{T}  \E_{\ba(t)} \left|[X_{j}(t)\mid \bar \bef_t] - [X_{j}(t)\mid \bef]\right| + \mcB(\mcF,\delta) \ .\label{equ:lemmasum2_mid}
\end{align}
where the transition to \eqref{equ:mid_of_lemma1} holds due to the triangular inequality via adding and subtracting terms $\bar  f_{i,t}\left( [X_{\Pa(i)}(t)\mid \bef] \f a_i(t) \right)$; \eqref{equ:mid_of_lemma} holds due to Lipschitz-continuity in conjunction with Lemma~\ref{lemma:bounddelta}; and \eqref{equ:lemmasum2_mid} holds since the triangle inequality of $L_2$ norm.

Next, we find an upper bound on the first summand in \eqref{equ:lemmasum2_mid}. We notice that the summation is taken over all parents of node $i$, thus, we aim to find an upper bound for the error bound for each parent. Based on the induction assumption, for each node $j\in\Pa(i)$, we have 
\begin{align}
    \sum_{t=1}^{T}\E_{\ba(t)} \left|[X_{j}(t) \mid \bar \bef_t] - [X_{j}(t) \mid \bef]\right| & \;  \leq \; \mcB(\mcF,\delta)\sum_{\ell=1}^{L_j}  d^{\ell-1} \prod_{k=2}^{\ell} K^{(k)} \\
    &     \label{equ:inductioncondition} \;  \leq \;  \mcB(\mcF,\delta)  \sum_{\ell=1}^{m} d^{\ell-1} \prod_{k=2}^{\ell} K^{(k)}  \ ,
\end{align}
where \eqref{equ:inductioncondition} holds due to the causal depth of parent nodes is less than that of node $i$, i.e., $L_j\leq L_i-1 = m$ for all $j\in\Pa(i)$. Subsequently, plugging \eqref{equ:inductioncondition} into \eqref{equ:lemmasum2_mid}, we obtain
\begin{align}
    K_i \sum_{j\in\Pa(i)}  \sum_{t=1}^{T} \E_{\ba(t)} \left|[X_{j}(t)\mid \bar \bef_t] - [X_{j}(t)\mid \bef]\right|  & \;  \leq \; 
    K_i\; \sum_{j\in\Pa(i)} \mcB(\mcF,\delta)  \sum_{\ell=1}^{m}  d^{\ell-1} \prod_{k=2}^{\ell} K^{(k)} \\
    & \;  \leq \;  d K_i\; \mcB(\mcF,\delta)  \sum_{\ell=1}^{m}  d^{\ell-1} \prod_{k=2}^{\ell} K^{(k)} \\
    &\leq \mcB(\mcF,\delta)\sum_{\ell=1}^{m}  d^{\ell} \prod_{k=2}^{\ell+1} K^{(k)}\label{equ:lemma5.2sum2}\ ,
\end{align}
where the inequality in \eqref{equ:lemma5.2sum2} follows from $K_i\leq K^{(m+1)}$. Combining  the results in~\eqref{equ:lemmasum2_mid}
and~\eqref{equ:lemma5.2sum2}, we conclude
\begin{align}
\sum_{t=1}^{T}\E_{\ba(t)}\left|[X_{i}(t) \mid \bar {\bef}_t] - [X_{i}(t) \mid \bef]\right|  &\leq \mcB(\mcF,\delta)  \sum_{\ell=1}^{m+1} d^{\ell-1} \prod_{k=2}^{\ell} K^{(k)} \ ,
\end{align}
which establishes the desired property in~\eqref{equ:lemma5aim}. Finally, by applying Jensen's inequality we find
\begin{align}
    \sum_{t=1}^{T}\left|\E_{\ba(t)}[X_{i}(t) \mid \bar {\bef}_t] -  \E_{\ba(t)}[X_{i}(t)\mid \bef] \right| &  \leq \E_{\ba(t)} \left| \sum_{t=1}^{T} [X_{i}(t) \mid \bar{\bef}_t] -  [X_{i}(t)\mid \bef] \right| \\
    & \leq \mcB(\mcF,\delta) \sum_{\ell=1}^{L_i}  d^{\ell-1} \prod_{k=2}^{\ell} K^{(k)} \ ,
\end{align}
which completes the proof.

\subsection{Proof of Theorem~\ref{thm:regret}}
We start by setting $\delta=\frac{1}{NT}$ and defining the event in which over $T$ rounds, all confidence sets $\mcC_{i,t}$ contain the ground truth function $f_i$. Specifically, 
\begin{equation}
    \mcE_i \triangleq \left\{f_i \in \mcC_{i,t}\ , \;\; \forall t\in[T]\right\}\ .
\end{equation}
Accordingly, define the event
\begin{align}
    \mcE \triangleq \bigcap_{i=1}^N \mcE_i\ .
\end{align}
According to Lemma~\ref{lemma:concentration},  we know $\P(\mcE^{\rm c}_i)\leq \frac{2}{NT}$. By invoking the union bound on probability, we have
\begin{equation}
    \P(\mcE^{\rm c}) \;  \leq \; \sum_{i=1}^{N} \P(\mcE^{\rm c}_i) \; \leq \; \sum_{i=1}^{N} \frac{2}{NT} \; = \; \frac{2}{T} \ .
\end{equation}
Next, we decompose the regret defined in \eqref{eq:frequentist_regret}  under the events $\mcE$ and $\mcE^{\rm c}$.
\begin{align}
    \E[R(T)] &= \sum_{t=1}^{T} \left[\mu_{\ba^*}-\mu_{\ba(t)}\right]\\
    & =  \sum_{t=1}^{T} \E_{\ba^*}[X_{N}\mid  \bef] - \E_{\ba(t)}[X_{N}\mid \bef] \\
    \nonumber & =  \sum_{t=1}^{T} \E \Bigg[ \mathds{1}\{\mcE^{\rm c}\} \underset{\leq 2C_N}{\underbrace{\bigg(\E_{\ba^*}[X_{N}\mid  \bef] - \E_{\ba(t)}[X_{N}\mid \bef]\bigg)}}\Bigg] + \sum_{t=1}^{T}  \E \Bigg[ \mathds{1}\{\mcE\}\;\bigg(\underset{\leq \operatorname{UCB}_{\ba^*}(t)}{\underbrace{\E_{\ba^*}[X_{N}\mid  \bef]}} - \E_{\ba(t)}[X_{N}\mid \bef]\bigg)\Bigg]\\
    & \leq 2C_N \sum_{t=1}^{T} \underset{\leq \frac{2}{T}}{\underbrace{\P (\mcE^{\rm c})}}
      + \sum_{t=1}^{T}  \E \Bigg[ \operatorname{UCB}_{\ba(t)}(t) - \E_{\ba(t)}[X_{N}\mid \bef]\Bigg]\\
    & \leq 4C_N + \sum_{t=1}^{T} \E \bigg[ \E_{\ba(t)}[X_{N}\mid \bar \bef_t] - \E_{\ba(t)}[X_{N}\mid \bef]\bigg]\label{eq:upperbound_plugin}\ ,
\end{align}
where we have used the set of inequalities 
\begin{equation}
\E_{\ba^*}[X_N \mid \bef ] \leq \operatorname{UCB}_{\ba^*}(t)\leq \operatorname{UCB}_{\ba(t)}(t_)\leq \E_{\ba(t)}[X_N \mid \bar \bef_t] \ ,
\end{equation}
in which $\bar \bef_t$ is a function that achieves the upper bound $\operatorname{UCB}_{\ba(t)}(t)$, i.e., 
\begin{equation}
    \bar{\bef}_t \in \argmax_{\bef\in \mcC_t} \operatorname{UCB}_{\ba(t)}(t) \ .
\end{equation}
 Then using the lemma~\ref{lm:bound_l_paths}, we have 
\begin{align}
    \E[R(T)]  &\leq 4C_N+\mcB(\mcF,\delta) \sum_{\ell=1}^{L} d^{\ell-1}\prod_{k=2}^{\ell} K^{(k)}\\
    & = \mcO\Big(\prod_{\ell=2}^{L} K^{(\ell)}d^{L-1} \sqrt{T\; \operatorname{dim}(\mcF) \log\left(NT\cn(\mcF)\right)}\Big)\label{eq:boundregret_1} \\
    & = \mcO\Big(Kd^{L-1} \sqrt{T\; \operatorname{dim}(\mcF) \log\left(NT\cn(\mcF)\right)}\Big)\label{eq:boundregret_2}\ ,
\end{align}
where the~\eqref{eq:boundregret_1} is due to the fact that we have $\alpha_i\leq \frac{1}{T}$ and the fact that for $d\geq 2$
\begin{equation}
    \sum_{\ell=1}^{L} d^{\ell-1}=\frac{d^{L}-1}{d-1}   \leq 2d^{L-1} = \tilde{\mcO} (d^{L-1})\ ,
\end{equation}
and \eqref{eq:boundregret_2} holds due to the definition $K=\prod_{\ell=2}^{L} K^{(\ell)}$.


\section{Regret Upper Bounds for Special SCMs}

We start by providing an equivalent definition of eluder dimension $\operatorname{dim}(\mcF_i,\epsilon)$, which follows the definitions of $\epsilon$-dependence and $\epsilon$-eluder dimension in Definition~\ref{def:dependence} and Definition~\ref{def:eluder}.  For a class of functions $\mcF_i$ with input space $\mcZ_i$, an ordered sequence of $m$ inputs $\mcZ_i(n) = \{Z_i(m)\in\mcZ_i:m\in[n]\}$, and $\epsilon\in\R^+$, for some $\epsilon'\geq \epsilon$ we define
\begin{equation}\label{equ:def:w_k}
    W_s(\mcF_i , \mcZ_i(n))\triangleq \sup\left\{\left|f(Z_i(s))-\tilde f(Z_i(s))\right| ~\Bigg|~ f,\tilde f\in\mcF_i, \; \sqrt{\sum_{m=1}^{s-1}\big(f(Z_i(m))-\tilde f(Z_i(m))\big)^2}\leq \epsilon'\right\}\ .
\end{equation}
The eluder dimension $\operatorname{dim}(\mcF_i,\epsilon)$ is the length of the longest ordered sequence $\mcZ_i(n)$ such that we have $ W_s(\mcF_i , \mcZ_i(n)) \geq \epsilon'$ for all $t\in[n]$. Throughout the rest of this section, when it is clear from the context, we suppress the dependence of $ W_s(\mcF_i , \mcZ_i(n)))$ on $\mcZ_i(n)$ and $\mcF_i$, and use the shorthand $W_s$.

\label{sec:refinedupperbounds}
\subsection{Proof of Theorem~\ref{th:SCM_lin} (Linear SCMs)}

\begin{lemma}
\label{lem:linearSEM_dim}
    Consider the class of linear SCMs $\{\mcF_i:i\in[N]\}$ specified in~\eqref{eq:SCM_lin} with the binary intervention space $\mcA_i=\{0,1\}$. The eluder dimension of $\mcF_i$ satisfies 
  \begin{equation}
    \operatorname{dim}(\mcF_i)  = \tilde{\mcO}(d_i\log(T))\ .
\end{equation}
\end{lemma}
\label{eluder:linear}
\begin{proof} 
For the selected class, we have
\begin{equation}
    f_i(X_{\Pa(i)}\; ;\; a_i)=\langle \theta_i(1-a_i)+a_i\bar\theta_i\; , \;  X_{\Pa(i)}\rangle\ .
\end{equation}
Consider two functions $f_i,\tilde f_i\in\mcF_i$ specified by model parameters $(\theta_i, \bar \theta_i)$ and $(\psi_i, \bar \psi_i)$. Accordingly, define
\begin{align}
    \Delta = \theta_i - \psi_i\ , \qquad \mbox{and} \qquad \bar \Delta = \bar \theta_i - \bar \psi_i\ .
\end{align}
Leveraging the norm bound on the model parameters specified in~\eqref{eq:SCM_lin}, we have $\norm{\Delta}\leq 2K_i$ and $\norm{\bar\Delta}\leq 2K_i$. We also define the following Gram matrices:
\begin{align}
\label{eq:Phi_0} \Phi_{s,0}\triangleq  \sum_{m=1}^{s-1}\mathds{1}\{a_i(m)=0\}X_{\Pa(i)}(m) X_{\Pa(i)}^{\top}(m)\ ,\\
\label{eq:Phi_1} \Phi_{s,1}\triangleq \sum_{m=1}^{s-1}\mathds{1}\{a_i(m)=1\}X_{\Pa(i)}(m) X_{\Pa(i)}^{\top}(m)\ .
\end{align}
By setting
\begin{align}
    \lambda\triangleq \Big(\frac{\epsilon'}{2K_i}\Big)^2\ ,
\end{align}
we also define the $d_i\times d_i$ matrices
\begin{align}\label{eq:V}
V_{t,0}\triangleq \Phi_{t,0}+\lambda I\ , \qquad \mbox{and} \qquad V_{t,1} \triangleq \Phi_{t,1}+\lambda I\ .
\end{align}
Based on \eqref{eq:Phi_0} and \eqref{eq:Phi_1} we also have
\begin{align}
    \sum_{m=1}^{s-1} \left(f_i(X_{\Pa(i)}(m) \f a_i(m))-\tilde f_i(X_{\Pa(N)}(m) \f a_i(m) )\right)^2=\Delta^{\top}\Phi_{s,0} \Delta + \bar \Delta^{\top}\Phi_{s,1}\bar \Delta\ .
\end{align}
The proof involves bounding the number of times that the event $W_s>\epsilon'$ occurs for any ordered sequence $\{Z_i(m)\in\mcZ_i:m\in[n]\}$, and it consists of three steps.

\textbf{Step 1:} We start by showing that if for a given $t\in \N$, the event  $W_t\geq \epsilon'$ implies that 
\begin{equation}
    X_{\Pa(i)}^{\top}(t) V_{t,a_i(t)}^{-1}X_{\Pa(i)}(t)\geq \frac{1}{2} \ .
\end{equation}
We provide the proof for $a_i(t)=0$, and that of $a_i(t)=1$ follows the same steps. We first find an upper bound on $W_t$ as follows.

    \begin{align}
        W_t& = \sup\left\{\left|\langle \Delta(1-a_i(t))+a_i(t) \bar\Delta , X_{\Pa(i)}(t)\rangle\right| ~\Bigg|~ \|\Delta\|\leq 2K_i, \|\bar\Delta\|\leq 2K_i , \;  \Delta^{\top}\Phi_{t,0}\Delta +   {\bar \Delta}^{\top}\Phi_{t,1}\bar \Delta \leq (\epsilon')^2\right\}\ .
        \label{equ:dimension_mid1}\\
        &= \sup\{|\langle \Delta, X_{\Pa(i)}(t)\rangle|~|~ \|\Delta\|\leq 2K_i, \|\bar\Delta\|\leq 2K_i , \Delta^{\top}\Phi_{t,0}\Delta +   {\bar \Delta}^{\top}\Phi_{t,1}\bar \Delta \leq (\epsilon')^2\} \label{equ:dimension_mid2}\\
        &\leq \sup\{|\langle \Delta, X_{\Pa(i)}(t)\rangle|~|~\Delta^{\top}\Phi_{t,0}\Delta \leq (\epsilon')^2\} \label{equ:dimension_mid2_1}\\
        &\leq \sup\{|\langle \Delta, X_{\Pa(i)}(t)\rangle|~|~\Delta^{\top}V_{t,0}\Delta \leq 2(\epsilon')^2\} \label{equ:dimension_mid3}\\
        & = \sqrt{2(\epsilon')^2}\|X_{\Pa(i)}(t)\|_{V_{t,0}^{-1}} \ ,
    \end{align}
    where we use the definition of $w_t$ in \eqref{equ:dimension_mid1}, we set $a_i(t)=1$ in \eqref{equ:dimension_mid2}, and we relax the conditions in  \eqref{equ:dimension_mid2_1} by dropping the norm constraints and noting that $\Phi_{t,1}$ is positive semidefinite.  The last equality in ~\eqref{equ:dimension_mid3} holds due to $\Delta^{\top} \lambda I \Delta \leq (\epsilon')^2$ and we relax the condition. Hence, the event $W_t\geq \epsilon'$ implies that
    \begin{equation}
        \|X_{\Pa(i)}(t)\|_{V_{t,0}^{-1}}\geq \frac{1}{2} \ .
    \end{equation}
    
\textbf{Step 2:} In this step, we provide upper and lower bounds for the determinants of matrices $V_{t,0}$ and $V_{t,1}$. For this purpose, let us define the counting numbers $N_0(t)$ and $N_1(t)$ as the numbers that count the number of interventions $0$ and $1$ chosen up to time, i.e., 
$t-1$, respectively.
\begin{align}
    N_0(t) = \sum_{s=1}^{t-1}\mathds{1}\{a_i(s)=0\} \ , \qquad  \mbox{and} \qquad 
    N_1(t) = \sum_{s=1}^{t-1} \mathds{1}\{a_i(s)=1\} \ .
\end{align}
Then, if $W_s\geq \epsilon'$ holds for all $s\in[t]$ we have the following inequalities.
\begin{align}
    \operatorname{det}(V_{t,0})\geq \lambda^{d_i} \left(\frac{3}{2}\right)^{N_0(t)} , \qquad \mbox{and} \qquad 
    \operatorname{det}(V_{t,1})\geq \lambda^{d_i} \left(\frac{3}{2}\right)^{N_1(t)} \ .
\end{align}
From the definition of $V_{t,0}$ in \eqref{eq:V}, for all $s\in[t]$ we have 
\begin{equation}
    V_{s,0} = V_{s-1,0}+\mathds{1}\{a_i(s-1) = 0\} X_{\Pa(i)}(s-1) X_{\Pa(i)}^{\top}(s-1)\ .
\end{equation} By using the matrix determinant lemma, we obtain
\begin{align}\label{eq:det_rec}
\operatorname{det} (V_{s,0}) &= \left\{ \begin{aligned}
    &\operatorname{det} (V_{s-1,0})\left(1+ \|X_{\Pa(i)}(s-1)\|_{V_{s-1,0}^{-1}}\right) \geq  \frac{3}{2} \operatorname{det}(V_{s,0}) &\text{ if } a_i(s-1) = 0 \ \\
    &\operatorname{det} (V_{s-1,0}) &\text{ if } a_i(s-1) = 1 \ 
\end{aligned}\right. \ ,
\end{align}
where for the case of $a_i(s-1)=0$ we use the results in {Step 1} stating  that $\|X_{\Pa(i)}(s)\|_{V_{s,0}^{-1}}\geq \frac{1}{2}$ for all $s\in[t]$. By noting that for the initial value of $s=0$ we have $\operatorname{det}(V_{0,0})=\operatorname{det}(\lambda I) = \lambda^{d_i}$, by recursively applying \eqref{eq:det_rec} we have
\begin{align}\label{eq:det_LB}
    \operatorname{det}(V_{t,0})\geq \lambda^{d_i} \left(\frac{3}{2}\right)^{N_0(t)}  \ .
\end{align}
The results for $V_{t,1}$ follow the same line of arguments. Since matrices $V_{t,0}$ and $V_{t,1}$ are symmetric and positive-definite, we have the following arithmetic-geometric inequality and leveraging~\eqref{eq:V} we have
\begin{equation}\label{eq:det_UB}
\operatorname{det}(V_{t,0}) \leq \left(\frac{\operatorname{tr}(V_{t,0})}{d_i}\right)^{d_i} \leq\left(\frac{N_0(t)}{d_i}c+\lambda\right)^{d_i} \ , \quad \mbox{and} \quad \operatorname{det}(V_{t,1}) \leq \left(\frac{\operatorname{tr}(V_{t,1})}{d_i}\right)^{d_i} \leq\left(\frac{ N_1(t)}{d_i}c+\lambda\right)^{d_i} \ .
\end{equation}
where we have set $c=\sup_{X_{\Pa(i)}}\norm{X_{\Pa(i)}}^2$.
Next, recall that $\operatorname{det}(V_{t,0})$ is the product of the eigenvalues of $V_{t,0}$, whereas the trace $\operatorname{tr}(V_{t,0})$ is the summation of the eigenvalues. As  $\operatorname{det}(V_{s,0})$ is positive definite, the determinant is maximized when all eigenvalues are equal

\textbf{Step 3:} Based on the determinant bounds in~\eqref{eq:det_LB} and \eqref{eq:det_UB}, we can find the following upper bound on $N_0(t)$:
\begin{equation}
\label{eq:eluderlin1}
    \left(\frac{3}{2}\right)^{\frac{N_0(t)}{d_i}}\leq  \frac{N_0(t)}{d_i}\bar c  +1 \ ,
\end{equation}
where we have defined $\bar c=c/\lambda$.
To find an upper bound on $N_0(t)$, we find the largest $N_0(t)$ that satisfies \eqref{eq:eluderlin1}. For this purpose, we define the function 
\begin{equation}
    B(x) \triangleq  \max\{y \geq 1 \; : \; (1+x)^y\leq \bar c y +1\}\ ,
\end{equation}
based on which we have  $N_0(t)\leq d_i B(\frac{1}{2})$. Next, we find an explicit upper bound on $B(x)$ for $x\leq 1$. By noting that $\log(1+x)\geq \frac{x}{1+x}$ for $x\leq 1$, from the definition of $B(x)$ we have 
\begin{align}
     B(x)\cdot \frac{x}{1+x } &\leq B(x)\log(1+x) \\
     &\leq \log( \bar c\bar B(x)+1)\\
    &\leq \log(\bar c B(x) + \bar B(x))\\
    &= \log B(x) + \log(1+\bar c)\ \\
    &= \log \left(B(x)\cdot \frac{x}{1+x }\right) +\log\frac{1+x}{x} + \log(1+\bar c)\\
    &\leq B(x)\cdot \frac{x}{1+x }\cdot\frac{1}{e} +\log\frac{1+x}{x} + \log(1+\bar c)\ .
\end{align}
By rearranging the terms, we have
\begin{align}
     B(x)\cdot \frac{x}{1+x } \left(1-\frac{1}{e}\right) \log\frac{1+x}{x} + \log(1+\bar c)\ ,
\end{align}
or equivalently
\begin{equation}
    B(x)\leq \frac{1+x}{x}\cdot \frac{e}{e-1}\left(  \log\frac{1+x}{x} + \log(1+\bar c)\right) \ ,
\end{equation}
which specifies an upper bound on $B(x)$ in terms of $x$. By setting $x=\frac{1}{2}$ we obtain 
\begin{equation}
    N_0(t) \leq d_iB\left(\frac{1}{2}\right) \leq 4.8\Big(1.1+\log (1+\bar c)\Big) d_i\ .
\end{equation}
Following the same steps, we can find a similar bound for $N_1(t)$, based on which by recalling the definitions $\bar c=c/\lambda$, $\lambda=(\epsilon'/2K_i)^2$, and by setting $\epsilon'=\frac{1}{T}$ we obtain 
\begin{equation}
    \operatorname{dim}(\mcF_i, \epsilon) \leq N_0(t)+ N_1(t) \leq 9.6\Big(1.1+\log (1+\bar c)\Big)d_i = \tilde{\mcO}(d_i\log(T))\ .
\end{equation}
\end{proof}

\begin{lemma}[Covering number.]
\label{lem:linearSEM_cn}
    Consider the class of linear SCMs $\{\mcF_i:i\in[N]\}$ specified in~\eqref{eq:SCM_lin} with the binary intervention space $\mcA_i=\{0,1\}$. The covering number of $\mcF_i$ satisfies 
  \begin{equation}
    \log \cn(\mcF_i)  = \tilde{\mcO}(d\log(T))\ .
\end{equation}
\end{lemma}
\begin{proof}
    We adopt the following lemma (described in terms of the parameters in this paper) to find an upper bound on the covering number for linear SCMs. 
\begin{lemma}\citep{russo2013eluder} \label{lem:coveringnumber}
Suppose the parameters of the Linear SCM class $\mcF_i$ belong to the hypercube $[-A,A]^{2d_i}$. Then $\cn_{\alpha}\left(\mcF_i\right) \leq(1 + 2Ac/\alpha)^{2d_i}$. 
\end{lemma}

By setting $A=K_i$, our hyper-sphere parameter space $\|\theta_i\|\leq K_i$ and $\|\bar \theta_i\|\leq K_i$ will be confined within the hypercube $[\theta_i,\bar \theta_i] \subset[-K_i,K_i]^{2d_i}$. Subsequently, the associate covering number will be upper-bounded by $(1 + 2cK_i/\alpha)^{2d_i}$. By setting $\alpha=\frac{1}{T}$ we have
\begin{equation}
    \ln \cn_{\alpha}(\mcF_i) \leq 2d_i\ln \left(1+2cK_iT\right) =  \tilde{\mcO}(d\log(T))\ .
\end{equation}
\end{proof}
Based on the results of Lemma~\ref{lem:linearSEM_dim} and Lemma~\ref{lem:linearSEM_cn}, we the regret bound in Theorem~\ref{thm:regret} becomes
\begin{align}
    \E[R(T)] &= \mcO\left(K d^{L-1}  \sqrt{T\; \operatorname{dim}(\mcF) \log \big( NT \cn(\mcF) \big) }\right) \\
    &=\mcO\left(K d^{L-1}  \sqrt{T\; d \log T \Big(\log NT + d \log T\Big)  }\right) \\
    & = \mcO(Kd^{L}\sqrt{T\log T\varpi_{\rm L}})\ ,
\end{align}
where 
\begin{align}\label{eq:varpi_d}
    \varpi_{\rm L}=\log T \left[1+\frac{\log N}{d\log T}\right]\ .
\end{align}
Note that the term $d$ in \eqref{eq:varpi_d} can be refined to $d+1$, but that will not change the scaling behavior of the regret with respect to $T$.

The proofs presented for Lemma~\ref{lem:linearSEM_dim} and Lemma~\ref{lem:linearSEM_cn} can be readily extended to discrete and finite spaces. The main change in the proof pertains to repeating the arguments we presented for bounding $N_0(t)$ and $N_1(1)$, for all other interventions as well. This results in bounding the eluder dimension as  
\begin{equation}
    \operatorname{dim}(\mcF_i)=\sum_{j=1}^{|\mcA_i|} N_{j}(t) \leq |\mcA_i| 4.8\Big(1.1+\ln (1+\bar c)\Big) d_i= \tilde{\mcO}(d_i\log(T))\ .
\end{equation}
Since the $\mcA_i$ is finite, we can find a constant $c_a$ such that $|a_{i,j}|\leq c_a$  covering number is minor by replacing $c$ with $c+c_a$, which yields
\begin{equation}
    \ln \cn_{\alpha}(\mcF_i) \leq 2d_i\ln \left(1+2(c+c_a)K_iT\right) =  \tilde{\mcO}(d\log(T))\ .
\end{equation}

\subsection{Proof of Theorem~\ref{th:SCM_poly} (Polynomial SCMs)}
\label{eluder:polynomial}
For $p=2$ for the functions specified by \eqref{eq:functionpoly} we have
\begin{align}
        f_i(X_{\Pa(i)}\; ;\; a_i)=\left(\langle\theta_{i}\; , \;  [X_{\Pa(i)} , a_i] \rangle\right)^2 \ ,
\end{align}
which can be recast as 
\begin{equation}
    f_i(X_{\Pa(i)}\; ;\; a_i) = [X_{\Pa(i)} , a_i]^{\top} \Theta [X_{\Pa(i)} , a_i] \ ,
\end{equation}
where $\Theta\in\R^{(d+1)\times (d+1)}$ is a matrix with element $\Theta_{j,k} = \theta_{i,j}\theta_{i,k}$ with $\theta_{i,j}$ is $j$-th element of $\theta_i$. Hence, the associated class $\mcF_i$ specified by \eqref{eq:SCM_poly} is a subspace of the space of quadratic function. Hence, the eluder dimension of $\mcF_i$ can be upper bounded based on the following lemmas~\citep{osband2014model}.
\begin{lemma}
\label{lem:qua}
    For the class of polynomial SCMs defined in \eqref{eq:SCM_poly} with $p=2$ and bounded intervention space such that $|a_i|\leq 1$ for all $a_i\in\mcA$, the eluder dimension satisfies
    \begin{equation}
        \operatorname{dim} (\mcF_i,\epsilon) = \tilde{\mcO}\big((d+1)^2\log(1/\epsilon)\big) \ .
    \end{equation}
\end{lemma}
By noting that the set of parameters captured by $\Theta$ has $(d+1)^2$ degrees of freedom, thus we have
\begin{lemma}
    For the class of polynomial SCMs defined in \eqref{eq:SCM_poly} with $p=2$ and bounded intervention space such that $|a_i|\leq 1$ for all $a_i\in\mcA$, the covering number satisfies 
    \begin{equation}
        \log \cn_{\alpha}(\mcF_i) = \tilde \mcO \big((d+1)^2 \log(1/\epsilon)\big) \ .
    \end{equation}
\end{lemma}

Then with the choice of $\epsilon = \alpha_i=\frac{1}{T}$, we have
\begin{equation}
    \operatorname{dim} (\mcF_i) = \tilde{\mcO}\big((d+1)^2\log T\big) \ , \quad \mbox{and} \log \cn(\mcF_i) = \tilde \mcO \big((d+1)^2 \log T\big)\ .
\end{equation}

Then the regret bound in Theorem~\ref{thm:regret} becomes
\begin{align}
    \E[R(T)] = &\mcO\left(K d^{L-1}  \sqrt{T\; \operatorname{dim}(\mcF) \log \big( NT \cn(\mcF) \big) }\right) \\
    &=\mcO\left(K d^{L-1}  \sqrt{T\; (d+1)^2 \log T \log \big( NT \times T^{(d+1)^2} \big) }\right) \\
    & =\mcO(Kd^{L+1}\sqrt{T\log T\varpi_{\rm P}})\ ,
\end{align}
where
\begin{align}
\varpi_{\rm P}=\log T \left[1+\frac{\log N}{d\log T}\right] \ .    
\end{align}

\subsection{Proof of Theorem~\ref{th:SCM_nn} (Neural Network SCMs)}
Corresponding to the class of neural network SCMs specified by \eqref{eq:SCM_nn}, we augment the causal graph $\mcG$ to generate another causal graph $\tilde\mcG$ as follows. As shown in Figure~\ref{fig:NN}, for each node $i\in[N]$, We insert $s\geq d+1$ nodes between the set of parents of node $i$ and $i$ and add directed edges to all parents of $i$ to the $s$ intermediate nodes and from the $s$ intermediate nodes to node~$i$. Hence, graph $\tilde\mcG$ has $N+s|\{i: \Pa(i)\neq \emptyset\}|$ nodes, its maximum in-degree is~$s$, and its longest causal path has length $2L$. 
\begin{figure}
    \centering
    \includegraphics[height= 2.2in]{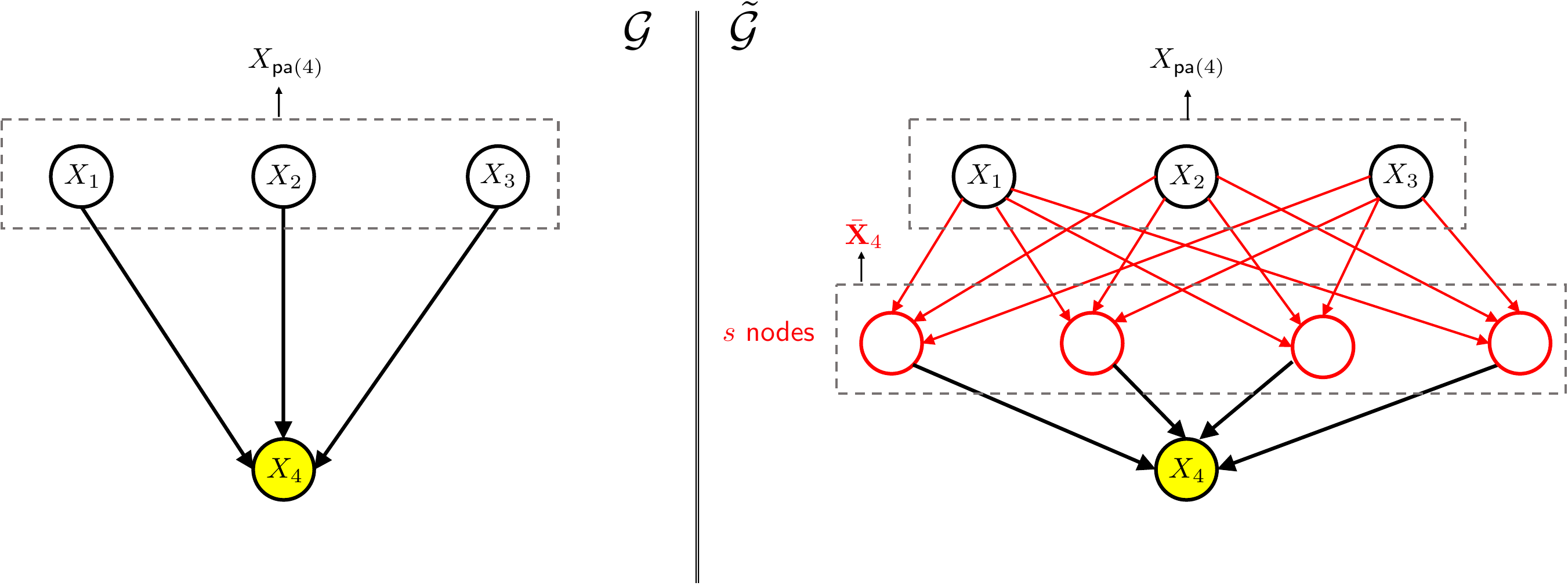}
    \caption{Construction of $\tilde\mcG$ based on $\mcG$ by embedding $s$ nodes between a node (node 4) and its parents (nodes 1,2,3).}
    \label{fig:NN}
\end{figure}

In the augmented graph $\tilde\mcG$, we denote the data of the $s$ intermediate nodes between $\Pa(i)$ and $i$ by $\bar\bX_i\in\R^s$ and for the data of node $i$, we keep the notation $X_i$. The SCM $\tilde\mcG$ follows a generalized linear model. Specifically consider node~$i$ with neural network parameters\footnote{We have different matrices $\bTheta_1$ and $\bTheta_2$ for different nodes, which we can specify by $\bTheta_{i,1}$ and $\bTheta_{i,2}$. However, we are suppressing the implicit dependence on $i$ for notational clarity.} $\bTheta_1$ and $\bTheta_2$. If node $i$ has an even causal depth, its SCM will be specified by $X_{i} = \sigma(\bTheta_{2}  \bar\bX_{i})$, which is a generalized linear model. On the other hand, if node $i$ has an odd causal depth, its SCM will be specified by $\bar \bX_{i}= \sigma(\bTheta_{1}  [a_i,X_{\Pa(i)}])$, which is also a generalized linear SCM. Note that the interventions will impact only the nodes with even causal depths. Their impact, subsequently, will be passed by the intermediate nodes to the nodes with even causal depths, which are the nodes in the original graph $\mcG$. Finally, for any node $i$, if it has an odd causal depth, we denote its Lipschitz constants by $K_{i,\rm o}$ and if it has an even causal depth, we denote its Lipschitz constants by $K_{i,\rm e}$. Accordingly, for even $i$ define
\begin{align}
    \bar K_i = s K_{i,\rm e} \cdot K_{i-1,\rm o}\ , \quad \forall i\in \{2\ell : \ell\in[N]\}\ ,
\end{align}
and
\begin{align}
\label{eq:lipshitznn}
   \bar K^{(k)} &= \max \{\bar K_i\;:\; L_i=2k \ , i \in \{2\ell : \ell\in[N]\}\}\  , \quad \forall k\in [L] \ ,
\end{align}
so that $\bar K^{(k)} \geq K^{(k)}$ is an upper bound of Lipschitz constant.

We can adjust the GCB algorithms to form estimates and confidence sets for all the parameters in the augmented graph $\tilde\mcG$. We denote the class of generalized SCMs that specify the cause-effect relationships in $\tilde\mcG$ by $\mcH$. Based on this, we provide a counterpart of Lemma~\ref{lm:bound_l_paths} for bounding the cumulative error at each node $i$ due to the compounding estimate uncertainties along the causal paths.

\begin{corollary}
[Compounding Error for Neural Networks]\label{lm:bound_l_paths_nn} Consider the neural network SCMs defined in \eqref{eq:SCM_nn}. If $f_i\in \mcC_{i,t}$ for all $i\in[N]$ and $t\in[T]$, then the following error bound holds 
\begin{align}
\label{equ:cumulative bound_nn}
    \sum_{t=1}^{T}&\Big|\E_{\ba(t)}[X_{i}\mid \bar \bef_t] - \E_{\ba(t)}[X_{i}\mid \bef]\Big| \leq \mcB(\mcH,\delta) \sum_{\ell=1}^{\frac{L_i}{2}}  (d+1)^{\ell-1} \prod_{k=2}^{\ell} \bar K^{(k)}\ .
\end{align}
\end{corollary}

\begin{proof}
This proof follows the flow of the one in Lemma~\ref{lm:bound_l_paths}, with the key difference being the addition of an intermediate layer between nodes $i$ and their parents. Consequently, each induction step is now comprised of two stages: firstly, upper bounding the error from the intermediate nodes $\bar\bX_i$ to $X_i$, and secondly, upper bounding the error from $X_{\Pa(i)}$ to the intermediate nodes. The second stage aligns with the procedure in the proof of Lemma~\ref{lm:bound_l_paths}, contributing to $(d+1) K_{j,{\rm o}}$ for node $X_j\in \bar\bX_i$. As for the contribution from $\bar\bX_i$ to $X_i$, it amounts to $s K_{i,{\rm e}}$. This results in the error bound presented in \eqref{equ:cumulative bound_nn}, with the Lipschitz constants $K^{(k)}$ defined in \eqref{eq:lipshitznn}. It is noteworthy that in the special case where $s=d+1$, our results reduce to the one in Lemma~\ref{lm:bound_l_paths} with a maximum causal depth of $2L$.
\end{proof}

Then plug \eqref{equ:cumulative bound_nn} into \eqref{eq:upperbound_plugin} in the proof of Theorem~\ref{thm:regret}, the upper bound in Theorem~\ref{thm:regret} becomes 
\begin{align}
    \E[R(T)]  &\leq 4C_N+\mcB(\mcH,\delta) \sum_{\ell=1}^{L}  (d+1)^{\ell-1}\prod_{k=2}^{\ell} \bar K^{(k)}\\
    & = \mcO\Big(\bar Kd^{L-1}  \sqrt{T\; \operatorname{dim}(\mcH) \log\left(NT\cn(\mcH)\right)}\Big)\label{eq:boundregretnn_2}\ .
\end{align}
Then, the eluder dimension and covering number of the class of functions $\mcH$ upper bounded by the following lemma.
\begin{lemma}~\citep{russo2013eluder}\label{lem:eludernn}
     For the generalized linear SCM class of functions $\mcH$ with degree $d$, we have
     \begin{align}
    \operatorname{dim}(\mcH) & = \tilde \mcO \big( r d \log T\big)\ , \qquad \mbox{and} \qquad   \log \cn(\mcH)  = \tilde \mcO (r d \log T)\ ,
    \end{align}
    where $r$ is defined as the ratio of maximum gradient over minimum derivative of the activation function\footnote{This can be generalized to the maximum and minimum derivative within the input space.}
    \begin{equation}
        r = \frac{\sup \nabla  \sigma(\cdot)}{ \inf  \nabla \sigma(\cdot)}\ .
    \end{equation}
\end{lemma}
Next, since the input dimension is at most $d+1$ for hidden layers and $s$ for the outputs $X_i$ for $i\in[N]$. Applying Lemma~\ref{lem:eludernn} in the result in \eqref{eq:boundregretnn_2}, we conclude the regret is upper bounded by
\begin{align}
    \E[R(T)] &= \mcO\Big( r \bar Kd^{L-1} \sqrt{T\; \operatorname{dim}(\mcH) \log\left(NT\cn(\mcH)\right)}\Big)\label{eq:boundregretnn_3}\\
    &=\mcO\Big(r \bar Kd^{L-1}\sqrt{T\log T\; \varpi_{\rm N}}\Big)\ ,
\end{align}
where 
\begin{equation}
    \varpi_{\rm N} = \log T \max{d,s} \left[1+\frac{\log N}{d\log T} \right] \ .
\end{equation}

\section{Regret Lower Bounds}
\label{sec:prooflower}

\begin{figure}[htb]
        \centering
        \includegraphics[height=3 in]{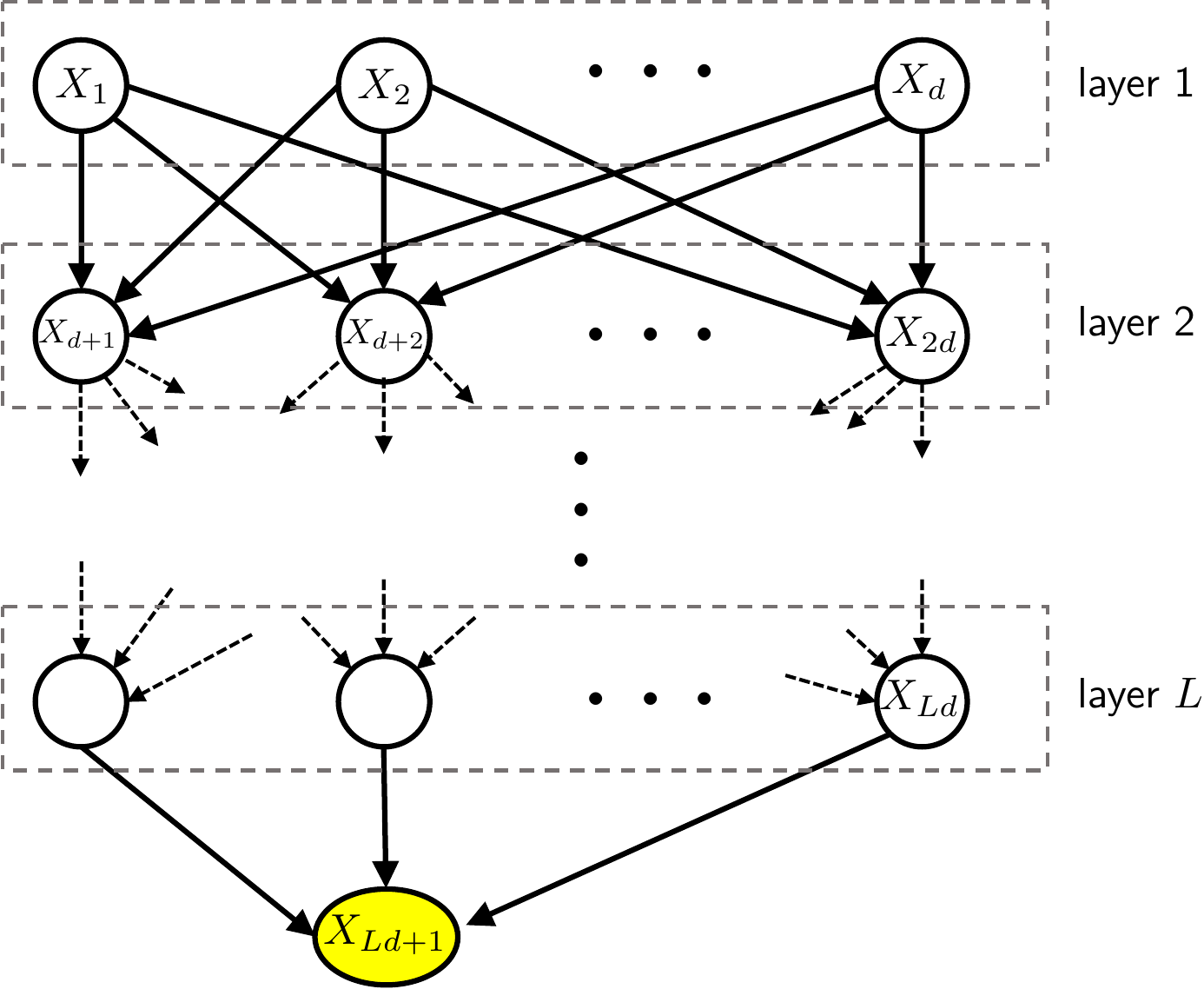}
        \caption{Hierarchical graph with degree $d$ and maximum causal path length $L$.}
        \label{fig:LB}
\end{figure}

{\bf Proof Sketch:} We find lower bounds for the term in~\eqref{eq:minimax} by lower bounding the inner supremum terms $\sup_{\mcG\in\bar\mcG} \E[R(T)]$, which captures the model with the worst-case regret. We can do this by shrinking the graph set $\bar\mcG$ to a set of only two possible graph instances $\{\mcG_1,\mcG_2\}$, for which the regret will not exceed that of the worst case regret over all possible choices in $\bar\mcG$. For this purpose, we choose a hierarchical graph shown in~Figure~\ref{fig:LB}. These two graphs consist of $L+1$ layers. Each of the first $L$ layers has $d$ nodes and the nodes within a layer are not connected. The nodes between two adjacent layers are fully connected. Finally, the last layer consists of one node (reward node) that is fully connected to the nodes in layer $L$. We use the bounded noise to ensure the boundness of the causal system. Subsequently, we provide an information-theoretic analysis that captures the hardness of distinguishing the two models, and facilitates finding a universal minimax lower bound for each class for each special family of models. Throughout the analysis, we will be using the following lemma \citep[Theorem 14.2]{lattimore2020bandit}.

\begin{lemma}[Bretagnolle-Huber Inequality]
\label{th:pinsker}
Let $\P_1$ and $\P_2$ be two probability measures on the same measurable space $(\Omega,\mcF)$ and let $A \in \mcF$ be an arbitrary event. Then,
\begin{align}
    \P_1(A) + \P_2(A^{\C}) \geq \frac{1}{2}\exp(-{\rm D}_{\rm KL}(\P_1 \kl \P_2)) \ ,
\end{align}
where ${\rm D}_{\rm KL}$ denotes the Kullback–Leibler divergence. 
\end{lemma}

\subsection{Proof of Theorem~\ref{thm:lowerlinear} (Linear SCMs)}
\label{sec:linear}
We consider two linear SCM causal bandit instances $\mcG_1$ and $\mcG_2$ that share the same hierarchical graph specified earlier (Figure~\ref{fig:LB}) with degree $d$ and maximum causal path length $L$. 
Subsequently, our minimax lower bound of interest is lower bounded by 
\begin{equation}
   \inf_{\pi\in\Pi} \sup_{\mcF\in\bar\mcF}\sup_{\mcG\in\bar\mcG} \E[R(T)] \geq \inf_{\pi\in\Pi} \sup_{\mcG\in\{\mcG_1,\mcG_2\}} \E[R(T) ]\ .
\end{equation}
We assume the two bandit instances share same weights $\{\theta_i: i\in[N]\}$ and $\{\bar \theta_i: i\in[N]\}$, which are set to
\begin{align}
    \theta_i  = \bigg[\underbrace{\frac{K^{(L_i)}}{\sqrt{d_i}}, \cdots, \frac{K^{(L_i)}}{\sqrt{d_i}}}_{d_i \text{ times}} \bigg]\ , \quad \mbox{and} \quad
    \bar \theta_i = \bigg[ \underbrace{\frac{K^{(L_i)}}{\sqrt{d_i}} - \delta, \cdots, \frac{K^{(L_i)}}{\sqrt{d_i}} - \delta}_{d_i \text{ times}} \bigg]\ ,
\end{align}
where $\delta>0$ is a small value to be chosen later. The intervention sets for two bandit instances $\mcG_1$ and $\mcG_2$ are set to binary, i.e. $\mcA_i=\{0,1\}$.
The difference between the two bandit instances is the noise distribution in the system. For the noise for the node $1$, in first bandit instance $\mcG_1$, it is defined as
\begin{equation}
    \epsilon_1 = \left\{\begin{aligned}
        &\operatorname{Bern}(1/2+\delta) &\text{ if } a_1=0\\
        &\operatorname{Bern}(1/2) &\text{ if } a_1=1
    \end{aligned}\right. \ ,
\end{equation}
where $\operatorname{Bern}(p)$ denotes the Bernoulli random distribution with probability $p$. The noise is reversed in the second bandit instance $\mcG_2$,  which is
\begin{equation}
    \epsilon_1 = \left\{\begin{aligned}
        &\operatorname{Bern}(1/2) &\text{ if } a_1=0\\
        &\operatorname{Bern}(1/2+\delta) &\text{ if } a_1=1
    \end{aligned}\right. \ .
\end{equation}
For the noise terms associated with the rest of the random variables, we assume that they follow the the following Bernoulli model.
\begin{equation}
    \epsilon_i = \left\{\begin{aligned}
        &-1 &\text{with probability } \frac{1}{2}\\
        &+1 &\text{with probability } \frac{1}{2}\\
    \end{aligned}\right. \ , \qquad \text{ for } i\neq 1 \ .
\end{equation}
Since the weights are non-negative and the noise terms are zero-mean for $i\neq 1$, the optimal action is the one that maximizes each weight and the average value of the node $1$. As a result, the optimal action at node $1$ is the action that maximizes the mean of noise, while the optimal actions at other nodes are the actions that maximize the weights, i.e., $a_i=0$ for $i\neq 1$. Hence, the optimal intervention for bandit instances $\mcG_1$ and $\mcG_2$ are 
\begin{align}
    \ba^*_{\mcG_1} = (0,0,\cdots, 0) \ , \quad \mbox{and} \quad
    \ba^*_{\mcG_2} = (1,0,\cdots, 0) \ .
\end{align}
Next, consider a fixed bandit policy $\pi$ that generates the filtration  $\mcH_t$ over time. The decision of $\pi$ at time $t$ is $\mcH_{t-1}$-measurable. Accordingly, define the $\P_{\mcG_1}$ and $\P_{\mcG_2}$ as the probability measures induced by $\mcH_T$ by $T$ rounds of interactions of $\pi$ with $\mcG_1$ and $\mcG_2$, respectively. We show that for any algorithm $\pi$, it cannot get small regret in both causal bandit instances simultaneously. Since the difference between two bandit instances is on the first node, we are interested in the event in which $a_1=1$ is chosen at least $\frac{T}{2}$ times after $T$ rounds of interactions. We denote this event by
\begin{equation}
    \mcE_{\pi} \triangleq \left\{\sum_{i=1}^{T}\mathds{1}\{a_1(t)=1\}\geq \frac{T}{2}\right\} \ .
\end{equation}
We observe that algorithm $\pi$ will incur a regret when choosing intervention with $a_1=1$ on bandit instance~$\mcG_1$. Furthermore, since the weights in the causal system are positive, we have $\mu_{\ba^*_{\mcG_2}}\geq \mu_{\ba}$ for all intervention $\ba$ with $a_1=1$. Therefore, for $\mcG_1$, the expected value of instantaneous regret is at least $\mu_{\ba^*_{\mcG_1}} - \mu_{\ba^*_{\mcG_2}}$ if $a_i(t)=1$. Besides, since $\epsilon_i$ has a zero mean for $i \geq 2$, only the paths that start at node $1$ and end at the reward node $N$ contribute to the expected regret. Thus, for the gap between $\mu_{\ba^*_{\mcG_1}}$ and $\mu_{\ba^*_{\mcG_2}}$, each path contributes $\prod_{\ell=1}^{L} K^{(\ell)} d^{-\frac{L}{2}} \delta$ and we have $d^{L-1}$ such paths. Therefore, we have
\begin{equation}
\label{eq:diff_mu}
    \mu_{\ba^*_{\mcG_1}} - \mu_{\ba^*_{\mcG_2}} =  d^{\frac{L}{2}-1}   \delta \prod_{\ell=1}^{L} K^{(\ell)}  \ .
\end{equation}
Consequently, we can lower bound the regret associated with $\mcG_1$ by conditioning on the event $\mcE_{\pi}$ being true, i.e., 
\begin{align}
    \E_{\mcG_1}[R(T)] &= \E_{\mcG_1}\left[\sum_{t=1}^{T} \left(\mu^* - \mu_{\ba(t)}\right)\right] \\ 
    &\geq   \E_{\mcG_1}\left[\sum_{t \in [T]: a_1(t)=1} \left(\mu_{\ba^*_{\mcG_1}} - \mu_{\ba(t)}\right)\right] \label{eqlowerbound_mid1}\\
    &\geq \E_{\mcG_1}\left[\sum_{t \in [T]: a_1(t)=1} \mathds{1}\{\mcE_{\pi}\}\left(\mu_{\ba^*_{\mcG_1}} - \mu_{\ba^*_{\mcG_2}}\right)\right] + \E_{\mcG_1}\left[\sum_{t \in [T]: a_1(t)=1} \mathds{1}\{\mcE_{\pi}^\C\}\left(\mu_{\ba^*_{\mcG_1}} - \mu_{\ba^*_{\mcG_2}}\right)\right]\label{eqlowerbound_mid1_1}\\
    & \geq  \P_{\mcG_1}(\mcE_{\pi})\frac{T}{2} \left(\mu_{\ba^*_{\mcG_1}} - \mu_{\ba^*_{\mcG_2}}\right) 
    \label{eqlowerbound_mid2}    \\ 
    &=  \P_{\mcG_1}(\mcE_{\pi}) \frac{T}{2}  d^{\frac{L}{2}-1}   \delta \prod_{\ell=1}^{L} K^{(\ell)} \label{eq:lower_bound_regret_1}\ ,
\end{align}
where \eqref{eqlowerbound_mid1} lower bound the regret by considering the time instance that regret incurred when $a_1(t)=1$; \eqref{eqlowerbound_mid1_1} holds due to $\mu_{\ba^*_{\mcG_2}}\geq \mu_{\ba}$ for all intervention $\ba$ with $a_1=1$ and we split the sum based on whether $\mcE_\pi$ holds; and \eqref{eqlowerbound_mid2} uses the definition of $\mcE_{\pi}$, lower bounding for the occurrence time of regret by $\frac{T}{2}$ for $\mcE_{\pi}$, and by $0$ for the $\mcE_{\pi}^{\C}$. The last equality in \eqref{eq:lower_bound_regret_1} holds due to equation \eqref{eq:diff_mu}.

Similarly, for bandit instance $\mcG_2$, we have $\mu_{\ba^*_{\mcG_1}} \geq \mu_{\ba}$ for all intervention $\ba$ that $a_i =0$. Therefore, for $\mcG_2$, the expected value of instantaneous regret is at least $\mu_{\ba^*_{\mcG_2}} - \mu_{\ba^*_{\mcG_1}}$ if $a_i(t)=0$. Hence, we can similarly show that 
\begin{align}
    \E_{\mcG_2}[R(T)] &\geq   \P_{\mcG_2}(\mcE^{\rm c}_{\pi}) \; \frac{T}{2}   d^{\frac{L}{2}-1} \delta \prod_{\ell=1}^{L} K^{(\ell)}  \ . \label{eq:lower_bound_regret_2}
\end{align}
Combining~\eqref{eq:lower_bound_regret_1} and~\eqref{eq:lower_bound_regret_2}, we have
\begin{equation}
    \E_{\mcG_1}[R(T)] + \E_{\mcG_2}[R(T)] \geq \left(\P_{\mcG_1}(\mcE_{\pi}) + \P_{\mcG_2}(\mcE_{\pi}^{\rm c})\right) \frac{T}{2}  d^{\frac{L}{2}-1} \delta \prod_{\ell=1}^{L} K^{(\ell)}\ ,
\end{equation}
Next, we leverage Lemma~\ref{th:pinsker} and find a lower bound on the term $\P_{\mcG_1}(\mcE_{\pi})+ \P_{\mcG_2}(\mcE_{\pi}^{\C})$, resulting in
\begin{align}
    \E_{\mcG_1}[R(T)] + \E_{\mcG_1}[R(T)] \geq \frac{T}{4}  d^{\frac{L}{2}-1} \delta \exp(-{\rm D}_{\rm KL}(\P_{\mcG_1} \kl  \P_{\mcG_2})) \prod_{\ell=1}^{L} K^{(\ell)} \ . \label{eq:lower_bound_regret_mid3}
\end{align}
It remains to compute the KL-divergence between two bandit instances $\mcG_1$ and $\mcG_2$, which is formalized by the following lemma.
\begin{lemma}
\label{lem:kllinear}
    For the causal bandit instance $\mcG_1$ and $\mcG_2$, the KL-divergence between $\P_{\mcG_1}$ and $\P_{\mcG_2}$ is upper bounded by
    \begin{equation}\label{eq:DL_bound}
        {\rm D}_{\rm KL}(\P_{\mcG_1} \kl \P_{\mcG_2}) \leq T \log\left(\frac{1}{(1+2\delta)(1-2\delta)}\right)\ .
    \end{equation}
\end{lemma}
\begin{proof}
    The two bandit instances differ only in the mechanism of node $1$, which is a root node. Hence, ${\rm D}_{\rm KL}(\P_{\mcG_1} \kl \P_{\mcG_1})$ can be simplified as
\begin{align}
    {\rm D}_{\rm KL}(\P_{\mcG_1} \kl \P_{\mcG_2}) = \sum_{i=1}^{N}  {\rm D}_{\rm KL}(\P_{\mcG_1}(X_i \med X_{\Pa(i)}) \kl  \P_{\mcG_2}(X_i \med X_{\Pa(i)})) = {\rm D}_{\rm KL}(\P_{\mcG_1}(X_1) \kl \P_{\mcG_2}(X_1)) \ ,
\end{align}
where the second equality holds because the noise terms for all nodes $i\geq 2$ are identical in $\mcG_1$ and $\mcG_2$. Next, we calculate ${\rm D}_{\rm KL}(\P_{\mcG_1}(X_1) \kl \P_{\mcG_1}(X_1))$ under two cases: when $a_1=0$ and $a_1=1$. By noting that the KL-divergence between two Bernoulli random variables is given by
\begin{align}
    {\rm D}_{\rm KL}(\operatorname{Bern}(p) \kl \operatorname{Ber(q)}) = p\log\Big(\frac{p}{q}\Big) + (1-p)\log\Big(\frac{1-p}{1-q}\Big) \ .
\end{align}

from the above we obtain
\begin{align}
    & {\rm D}_{\rm KL}(\P_{\mcG_1}(X_1) \kl \P_{\mcG_1}(X_1)) \notag \\ 
    &= \sum_{t \in [T]: a_i(t) = 0} {\rm D}_{\rm KL}(\operatorname{Bern}(1/2+\delta)\kl \operatorname{Ber(1/2)})   + \sum_{t \in [T]: a_i(t)=1}  {\rm D}_{\rm KL}(\operatorname{Bern}(1/2) \kl \operatorname{Ber(1/2+\delta)})  \label{eq:lowerlienr_mid_split} \\ 
    &= \sum_{s=1}^{T} \mathds{1}\{a_1(t)=0\} \; \left[\left(\frac{1}{2}+\delta\right) \log (1+2\delta) + \left(\frac{1}{2}-\delta\right) \log (1-2\delta)\right] \\
    &\qquad \qquad +  \sum_{s=1}^{T} \mathds{1}\{a_1(t)=1\} \; \frac{1}{2} \log\left(\frac{1}{(1+2\delta)(1-2\delta)}\right) \\ 
    & < \frac{T}{2} \log\left(\frac{1}{(1+2\delta)(1-2\delta)}\right)\ ,
\end{align}
where the last inequality holds since $(\frac{1}{2}+\delta) \log (1+2\delta) + (\frac{1}{2}-\delta) \log (1-2\delta) < \frac{1}{2} \log\left(\frac{1}{(1+2\delta)(1-2\delta)}\right)$ for $0<\delta<1/2$.
\end{proof}
If we choose $\delta=\frac{1}{\sqrt{T}}$ to balance the terms in the lower bound, we obtain
\begin{align}
    \max\{\E_{\mcG_1}[R(T)],\E_{\mcG_2}[R(T)]\} &\geq \frac{1}{2} \left(\E_{\mcG_1}[R(T)] + \E_{\mcG_2}[R(T)]\right)\\
     &\overset{\eqref{eq:lower_bound_regret_mid3}}{\geq} \frac{T}{8}  d^{\frac{L}{2}-1} \delta \exp(-{\rm D}_{\rm KL}(\P_{\mcG_1} \kl  \P_{\mcG_2})) \prod_{\ell=1}^{L} K^{(\ell)} \\
     &\overset{\eqref{eq:DL_bound}}{\geq} \frac{T}{8} d^{\frac{L}{2}-1} \delta  [(1+2\delta)(1-2\delta)]^{T} \prod_{\ell=1}^{L} K^{(\ell)}  \\
     & =\frac{1}{8}   d^{\frac{L}{2}-1} \sqrt{T} \times \Big(1-\frac{4}{T}\Big)^{T} \prod_{\ell=1}^{L} K^{(\ell)}\ .
\end{align}
for  $T\geq 5$, the term $\Big(1-\frac{4}{T}\Big)^{T}$ is an increasing function of $T$. Hence, for $T\geq 5$, we have the lower bound $\Big(1-\frac{4}{T}\Big)^{T}\geq 0.2^5$. By setting $c=\frac{1}{8}\times 0.2^5$, we have 
\begin{equation}
     \max\{\E_{\mcG_1}[R(T)],\E_{\mcG_2}[R(T)]\} 
     \geq c K d^{\frac{L}{2}-1}\sqrt{T} \ .
\end{equation}

Hence, the policy $\pi$ incurs a regret $c K d^{\frac{L}{2}-1}\sqrt{T}$ in at least one of the two bandit instances.

\subsection{Proof of Theorem~\ref{thm:lowerpoly} (Polynomial SCMs)}

\label{sec:poly}
Similarly to the linear SCM setting, we consider two polynomial SCM causal bandit instances $\mcG_1$ and $\mcG_2$ that share the same hierarchical graph specified earlier (Figure~\ref{fig:LB}) with degree $d$ and maximum causal path length $L$. Subsequently, the minimax regret of interest is lower bounded by
\begin{equation}
   \inf_{\pi\in\Pi} \sup_{\mcF\in\bar\mcF}\sup_{\mcG\in\bar\mcG} \E[R(T)] \geq \inf_{\pi\in\Pi} \sup_{\mcG\in\{\mcG_1,\mcG_2\}} \E[R(T) ]\ .
\end{equation}
We assume the two bandit instances share same weights $\{\theta_i: i\in[N]\}$, which are set to be the same $\sqrt{\beta}$ and our function in polynomial SCM becomes
\begin{equation}
    f_i(X_{\Pa(i)} \f a_i) = \beta \left( \sum_{j\in \Pa(i)}  X_{j} \right)^p\ ,
\end{equation}
where $\beta>0$ is a constant defined later that controls the scaling behavior of the system, which we control later. The intervention sets for two bandit instances $\mcG_1$ and $\mcG_2$ are set to $\mcA_i=[0,1]$. The difference between the two bandit instances is the noise distribution in the system. For the noise for the node $1$, in first bandit instance $\mcG_1$, it is defined as
\begin{equation}
    \epsilon_1 = \left\{\begin{aligned}
        &\operatorname{Bern}(1/2+\delta) &\text{ if } a_1 \leq \frac{1}{2}\\
        &\operatorname{Bern}(1/2) &\text{ if } a_1>\frac{1}{2}
    \end{aligned}\right. \ ,
\end{equation}
where $\delta\in(0,\frac{1}{2})$ is a small value to be chosen later. The noise is reversed in the second bandit instance $\mcG_2$,  which is
\begin{equation}
    \epsilon_1 = \left\{\begin{aligned}
        &\operatorname{Bern}(1/2) &\text{ if } a_1\leq \frac{1}{2}\\
        &\operatorname{Bern}(1/2+\delta) &\text{ if } a_1>\frac{1}{2}
    \end{aligned}\right. \ .
\end{equation}
We assume the rest of the random variables are noiseless, i.e. $\epsilon_i = 0$ for  $i\neq 1$. Otherwise, the system will be more difficult to learn and cause larger regret. Since the weights are fixed and non-negative and interventions at other nodes do not affect the reward, the optimal action is the one that maximizes the average value of the node $1$. In contrast, other interventions can be chosen freely. Hence, the optimal intervention set for $\mcG_1$ and $\mcG_2$ are
\begin{equation}
    \mcA_{\mcG_1}^* = \Big\{\ba \in \mcA, a_i\leq \frac{1}{2}\Big\} \quad \mbox{and} \ , \quad \mcA_{\mcG_2}^* = \Big\{\ba \in \mcA, a_i > \frac{1}{2} \Big\} \ .
\end{equation}
 We choose $\ba_{\mcG_1}^*\in\mcA_{\mcG_1}^*$ and $\ba_{\mcG_2}^*\in\mcA_{\mcG_2}^*$  to represent the optimal interventions. Since all the nodes $i\neq 1$ are noiseless, the variables in the same layer will have the same values. Then, we have the following recursive relationships in the system
\begin{equation}
    X_i = \beta \left(\sum_{j\in \Pa(i)}X_{j} \right)^{p} = \beta d^p X_{j}\ , \qquad \forall j\in \Pa(i) \ .
\end{equation}
Based on the recursive relationship, when we dissect the structure to the first layer from node $i\in[N]$, we get the distribution of $X_i$ as follows:
\begin{equation}
\label{eq:lowerXi}
    X_i = \left\{\begin{aligned}
        &0 &\text{ if } \epsilon_1=0 \\
        & \displaystyle \beta^{\frac{p^{L_i}-1}{p-1}} d^{\frac{p(p^{L_i-1}-1)}{p-1}} &\text{ if } \epsilon_1=1 
    \end{aligned}\right. \ .
\end{equation}
Accordingly, the Lipschitz constant, which is the maximum gradient magnitude, becomes
\begin{align}
    K_i &= \sup \Big\|\nabla_{X_{\Pa(i)}} \beta \big(\sum_{j\in \Pa(i)}X_{j} \big)^{p} \Big\|\\
    &= \sqrt{d} \beta p \; \sup\norm{X_j^{p-1}}\label{eq:lowerKi_1}\\
    &= \sqrt{d} \beta p \beta^{p^{L-1}-1} d^{p^{L_i-1}-p}\label{eq:lowerKi_2} \\
    &= p \beta^{p^{L}} d^{p^{L_i-1}-p+\frac{1}{2}}\ ,
\end{align}
where $j\in\Pa(i)$ is a fixed index, $\sqrt{d}$ term \eqref{eq:lowerKi_1} is due to the maximum gradient magnitude, which is maximized when 
\begin{align}
X_j=\beta^{\frac{p^{L_j}-1}{p-1}} d^{\frac{p(p^{L_j-1}-1)}{p-1}}\ , \qquad     \forall j\in\Pa(i)\ ,
\end{align}
and we plug in the characterization of $X_i$ specified in~\eqref{eq:lowerXi} in \eqref{eq:lowerKi_2}. Since the interventions on nodes other than node $1$ do not influence the reward, we have $\mu_{\ba^*_{\mcG_2}}= \mu_{\ba}$ for all intervention $\ba$ with $\mcA_{\mcG_2}^*$. Therefore, for $\mcG_1$, the expected value of instantaneous regret is $\mu_{\ba^*_{\mcG_1}} - \mu_{\ba^*_{\mcG_2}}$ if $a_1>\frac{1}{2}$. Furthermore, since the noises $\epsilon_i=0$ for $i\geq 2$, based on \eqref{eq:lowerXi}, we have
\begin{equation}
\label{eq:diff_mu_poly}
    \mu_{\ba^*_{\mcG_1}} - \mu_{\ba^*_{\mcG_2}} = \beta^{\frac{p^{L}-1}{p-1}} d^{\frac{p(p^{L-1}-1)}{p-1}}  \delta \ .
\end{equation}
By setting $\beta$ as
\begin{equation}
    \beta = \big(p^{L} d^{L(p-\frac{1}{2})}\big)^{-p^L-p} \ .
\end{equation}
we have the following relationship
\begin{equation}
\label{eq:diff_mu_poly_K}
    \mu_{\ba^*_{\mcG_1}} - \mu_{\ba^*_{\mcG_2}} = K \delta \ ,
\end{equation}
in which $K = \prod_{\ell=2}^{L} K^{(\ell)}$. Next, consider a fixed bandit policy $\pi$ that generates the filtration  $\mcH_t$ over time. The decision of $\pi$ at time $t$ is $\mcH_{t-1}$-measurable. Accordingly, define the $\P_{\mcG_1}$ and $\P_{\mcG_2}$ as the probability measures induced by $\mcH_T$ by $T$ rounds of interaction of $\pi$ with $\mcG_1$ and $\mcG_2$, respectively. We show that any algorithm $\pi$ cannot simultaneously achieve small regret in both causal bandit instances. Since the difference between the two bandit instances is caused by their differences in the first node, we are interested in the event in which $a_1>\frac{1}{2}$ is chosen at least $\frac{T}{2}$ times after $T$ rounds of interactions. We denote this event by
\begin{equation}
    \mcE_{\pi} \triangleq \left\{\sum_{i=1}^{T}\mathds{1} \Big\{a_1(t)> \frac{1}{2}\Big\}  \geq \frac{T}{2}\right\} \ .
\end{equation}
Consequently, we can lower bound the regret associated with $\mcG_1$ by conditioning on the event $\mcE_{\pi}$ being true, i.e.
\begin{align}
    \E_{\mcG_1}[R(T)] &= \E_{\mcG_1}\left[\sum_{t=1}^{T} \left(\mu^* - \mu_{\ba(t)}\right)\right] \\ 
    &\geq   \E_{\mcG_1}\left[\sum_{t \in [T]: a_1(t)>\frac{1}{2}} \left(\mu_{\ba^*_{\mcG_1}} - \mu_{\ba(t)}\right)\right] \label{eqlowerbound_mid1_poly}\\
    &\geq   \E_{\mcG_1}\left[\sum_{t \in [T]: a_1(t)>\frac{1}{2}} \mathds{1}\{\mcE_{\pi}\}\left(\mu_{\ba^*_{\mcG_1}} - \mu_{\ba^*_{\mcG_2}}\right)\right] + \E_{\mcG_1}\left[\sum_{t \in [T]: a_1(t)>\frac{1}{2}} \mathds{1}\{\mcE_{\pi}^{\C}\}\left(\mu_{\ba^*_{\mcG_1}} - \mu_{\ba^*_{\mcG_2}}\right)\right] \label{eqlowerbound_mid1_1_poly}\\
    &= \P_{\mcG_1}(\mcE_{\pi}) \frac{T}{2}\left(\mu_{\ba^*_{\mcG_1}} - \mu_{\ba^*_{\mcG_2}}\right)  \label{eqlowerbound_mid2_poly}\\ 
    &=  \P_{\mcG_1}(\mcE_{\pi}) \frac{T}{2}   K \delta \ ,\label{eq:lower_bound_regret_1_poly}
\end{align}
where \eqref{eqlowerbound_mid1_poly} is a lower bound on the regret by considering the regret incurred when $a_1(t)>\frac{1}{2}$; \eqref{eqlowerbound_mid1_1_poly} holds due to $\mu_{\ba^*_{\mcG_2}}= \mu_{\ba}$ for all intervention $\ba\in\mcA_{\mcG_2}^*$  and we split the sum based on whether $\mcE_{\pi}$ holds; and \eqref{eqlowerbound_mid2_poly} uses the definition of $\mcE_{\pi}$, lower bounding for the occurrence time of regret by $\frac{T}{2}$ for $\mcE_{\pi}$, and by 0 for the $\mcE_{\pi}^{\C}$. The last equality in \eqref{eq:lower_bound_regret_1_poly} holds due to \eqref{eq:diff_mu_poly_K}. Similarly, for bandit instance $\mcG_2$, we have $\mu_{\ba^*_{\mcG_1}} = \mu_{\ba}$ for all intervention $\ba\in\mcA_{\mcG_1}^*$. Therefore, for $\mcG_1$, the expected value of instantaneous regret is $\mu_{\ba^*_{\mcG_2}} - \mu_{\ba^*_{\mcG_1}}$ if $a_i(t)\leq \frac{1}{2}$. Hence, we can similarly show that
\begin{align}
    \E_{\mcG_2}[R(T)] \geq \P_{\mcG_2}(\mcE_{\pi}^{\rm c}) \frac{T}{2} K \delta \;  \ . \label{eq:lower_bound_regret_poly_2}
\end{align}
Combining~\eqref{eq:lower_bound_regret_1_poly} and~\eqref{eq:lower_bound_regret_poly_2}, we have
\begin{equation}
    \E_{\mcG_1}[R(T)] + \E_{\mcG_2}[R(T)] \geq \left(\P_{\mcG_1}(\mcE_{\pi}) + \P_{\mcG_2}(\mcE^{\rm c}_{\pi})\right) K \delta \ .
\end{equation}
Next, we use the Lemma~\ref{th:pinsker} and find a lower bound on the term $\P_{\mcG_1}(\mcE_{\pi})+ \P_{\mcG_2}(\mcE_{\pi}^{\C})$, resulting in
\begin{align}
    \E_{\mcG_1}[R(T)] + \E_{\mcG_1}[R(T)] \geq \frac{T}{4} K \delta \exp(-{\rm D}_{\rm KL}(\P_{\mcG_1} \kl  \P_{\mcG_2})) \ . \label{eq:lower_bound_regret_mid3_poly}
\end{align}
It remains to compute the KL-divergence between the two bandit instances $\mcG_1$ and $\mcG_2$, which is formalized by the following lemma.
\begin{lemma}
\label{lem:kllinear_poly}
    For the causal bandit instance $\mcG_1$ and $\mcG_2$, the KL-divergence between $\P_{\mcG_1}$ and $\P_{\mcG_2}$ is upper bounded by
    \begin{equation}\label{eq:DL_bound_poly}
        {\rm D}_{\rm KL}(\P_{\mcG_1} \kl \P_{\mcG_2}) \leq T \log\left(\frac{1}{(1+2\delta)(1-2\delta)}\right)\ .
    \end{equation}
\end{lemma}
\begin{proof}
The proof follows the same steps as that of Lemma~\ref{lem:kllinear} with the difference that we split the summation in \eqref{eq:lowerlienr_mid_split} by conditioning on the events $\{a_i(t)\leq \frac{1}{2}\}$ and $\{a_i(t) > \frac{1}{2}\}$. 
\end{proof}
If we choose $\delta=\frac{1}{\sqrt{T}}$ to balance the terms in the lower bound, we obtain
\begin{align}
    \max\{\E_{\mcG_1}[R(T)],\E_{\mcG_2}[R(T)]\} &\geq \frac{1}{2} \left(\E_{\mcG_1}[R(T)] + \E_{\mcG_2}[R(T)]\right)\\
     &\overset{\eqref{eq:lower_bound_regret_mid3_poly}}{\geq} \frac{T}{8} K \delta \exp(-{\rm D}_{\rm KL}(\P_{\mcG_1} \kl  \P_{\mcG_2})) \\
     &\overset{\eqref{eq:DL_bound_poly}}{\geq} \frac{T}{8} K \delta  [(1+2\delta)(1-2\delta)]^{T} \\
     & =\frac{1}{8} K \sqrt{T} \times \Big(1-\frac{4}{T}\Big)^{T} \ .
\end{align}
for $T\geq 5$, the term $(1-4/T)^{T}$ is an increasing function of $T$. Hence, for $T\geq 5$, we have the lower bound $(1-4/T)^{T}\geq 0.2^5$. By setting $c=\frac{1}{8}\times 0.2^5$, we have 
\begin{equation}
     \max\{\E_{\mcG_1}[R(T)],\E_{\mcG_2}[R(T)]\} 
     \geq c K \sqrt{T} \ .
\end{equation}

Hence, the policy $\pi$ incurs a regret $c K \sqrt{T}$ in at least one of the two bandit instances.

\subsection{Proof of Theorem~\ref{thm:lowernn} (Neural Network SCMs)}
\label{sec:NN}
Similarly to the linear and polynomial settings, we consider two neural network SCM causal bandit instances $\mcG_1$ and $\mcG_2$ that share the same hierarchical graph specified earlier (Figure~\ref{fig:LB}) with degree $d$ and maximum causal path length $L$. Subsequently, our minimax lower bound of interest is lower bounded by
\begin{equation}
   \inf_{\pi\in\Pi} \sup_{\mcF\in\bar\mcF}\sup_{\mcG\in\bar\mcG} \E[R(T)] \geq \inf_{\pi\in\Pi} \sup_{\mcG\in\{\mcG_1,\mcG_2\}} \E[R(T) ]\ .
\end{equation}
Within the class of neural network models, we use the parametric Rectified Linear Unit (ReLU) function as the activation function, which is specified by
\begin{equation}
    \sigma(x) = \left\{\begin{aligned}
        &\beta x &\text{ if } x\geq 0 \\
        &\alpha x &\text{ if } x<0 
    \end{aligned} \right. \ ,
\end{equation}
where the constant $\beta>\alpha>0$ control the slope of the function.  We assume the two bandit instances share same weights $\{\bTheta_{i,1} : i\in[N]\}$ and $\{\bTheta_{i,2}: i\in[N]\}$, which are specified as follows.

\begin{align}
[\bTheta_{i,1}]_{j,k} = \left\{
\begin{array}{ll}
    0\ , & k=1 \\
    &\\
     \displaystyle \frac{\sqrt{K^{(L_i)}}}{\beta \sqrt{s}} \ , & k\in[d_i+1], \; k\geq 2
\end{array}
\right. \ , \qquad \forall j\in[s]\ .
\end{align}
and
\begin{align}
    [\bTheta_{i,2}]_{j} &= \frac{\sqrt{K^{(L_i)}}}{\beta\sqrt{s}} \ , \qquad \forall j\in[s]\ .
\end{align}

The intervention sets for two bandit instances $\mcG_i$ and $\mcG_2$ are set to $\mcA_i=[0,1]$. The difference between the two bandit instances is the noise distribution in the system. For the noise for the node $1$, in first bandit instance $\mcG_1$, it is defined as
\begin{equation}
    \epsilon_1 = \left\{\begin{aligned}
        &\operatorname{Bern}(1/2+\delta) &\text{ if } a_1\leq \frac{1}{2}\\
        &\operatorname{Bern}(1/2) &\text{ if } a_1> \frac{1}{2}
    \end{aligned}\right. \ ,
\end{equation}
where $\delta\in(0,\frac{1}{2})$ is a small value to be chosen later. The noise is reversed in the second bandit instance $\mcG_2$,  which is
\begin{equation}
    \epsilon_1 = \left\{\begin{aligned}
        &\operatorname{Bern}(1/2) &\text{ if } a_1\leq \frac{1}{2}\\
        &\operatorname{Bern}(1/2+\delta) &\text{ if } a_1> \frac{1}{2}
    \end{aligned}\right. \ .
\end{equation}
We assume that the noise terms associated with the rest of the random variables follow the Bernoulli model.
\begin{equation}
    \epsilon_i = \left\{\begin{aligned}
        &-1 &\text{with probability } \frac{1}{2}\\
        &+1  &\text{with probability } \frac{1}{2}\\
    \end{aligned}\right. \qquad \text{ for } i\neq 1 \ .
\end{equation}
Since the weights are fixed and non-negative, and the noise terms are zero-mean when $i\neq 1$, the optimal action is the one that maximizes the average value of the node $1$, and the interventions at other nodes do not affect the reward. As a result, the optimal action at node $1$ is the action that maximizes the mean of noise, while other interventions can be chosen freely. Hence, the optimal intervention set for $\mcG_1$ and $\mcG_2$ are
\begin{equation}
    \mcA_{\mcG_1}^* = \Big\{\ba \in \mcA, a_i\leq \frac{1}{2}\Big\} \quad \mbox{and} \ , \quad \mcA_{\mcG_2}^* = \Big\{\ba \in \mcA, a_i > \frac{1}{2} \Big\} \ .
\end{equation}
 We pick $\ba_{\mcG_1}^*\in\mcA_{\mcG_1}^*$ and $\ba_{\mcG_2}^*\in\mcA_{\mcG_2}^*$  as the represent of the optimal intervention. Next, consider a fixed bandit policy $\pi$ that generates the filtration  $\mcH_t$ over time. The decision of $\pi$ at time $t$ is $\mcH_{t-1}$-measurable. Accordingly, define the $\P_{\mcG_1}$ and $\P_{\mcG_2}$ as the probability measures induced by $\mcH_T$ by $T$ rounds of interaction of $\pi$ with $\mcG_1$ and $\mcG_2$, respectively. We show that any algorithm $\pi$ cannot get small regret in both causal bandit instances simultaneously. Since the difference between two bandit instances is on the first node, we are interested in the event in which $a_1>\frac{1}{2}$ is chosen at least $\frac{T}{2}$ times after $T$ rounds of interactions. We denote this event by
\begin{equation}
    \mcE_{\pi} \triangleq \left\{\sum_{i=1}^{T}\mathds{1} \Big\{a_1(t)> \frac{1}{2}\Big\}  \geq \frac{T}{2}\right\} \ .
\end{equation}
Since interventions on nodes other than node $1$ do not influence the reward, we have $\mu_{\ba^*_{\mcG_2}}= \mu_{\ba}$ for all intervention $\ba\in \mcA_{\mcG_2}^*$. Therefore, for $\mcG_1$, the expected value of instantaneous regret is $\mu_{\ba^*_{\mcG_1}} - \mu_{\ba^*_{\mcG_2}}$ if $a_1>\frac{1}{2}$. Besides, since $\epsilon_i$ has a zero mean for $i \geq 2$, only the paths that start at node $1$ and end at the reward node $N$ contribute to the expected regret. Thus, for the gap between $\mu_{\ba^*_{\mcG_1}}$ and $\mu_{\ba^*_{\mcG_2}}$, each path contribute $\prod_{\ell=1}^{L} K^{(\ell)} d^{-\frac{L}{2}} s^{-L} \delta$ and we have $d^{L-1}s^L$ such paths. Therefore, we have
\begin{equation}
\label{eq:diff_mu_nn}
    \mu_{\ba^*_{\mcG_1}} - \mu_{\ba^*_{\mcG_2}} =   d^{\frac{L}{2}-1}  \delta \prod_{\ell=1}^{L} K^{(\ell)} \ .
\end{equation}
Consequently, we can lower bound the regret associated with $\mcG_1$ by conditioning on the event $\mcE_{\pi}$ being true, i.e.,
\begin{align}
    \E_{\mcG_1}[R(T)] &= \E_{\mcG_1}\left[\sum_{t=1}^{T} \left(\mu^* - \mu_{\ba(t)}\right)\right] \\ 
    &\geq   \E_{\mcG_1}\left[\sum_{t \in [T]: a_1(t)>\frac{1}{2}} \left(\mu_{\ba^*_{\mcG_1}} - \mu_{\ba(t)}\right)\right] \label{eqlowerbound_mid1_nn}\\
    &\geq   \E_{\mcG_1}\left[\sum_{t \in [T]: a_1(t)>\frac{1}{2}} \mathds{1}\{\mcE_{\pi}\}\left(\mu_{\ba^*_{\mcG_1}} - \mu_{\ba^*_{\mcG_2}}\right)\right] + \E_{\mcG_1}\left[\sum_{t \in [T]: a_1(t)>\frac{1}{2}} \mathds{1}\{\mcE_{\pi}^{\C}\}\left(\mu_{\ba^*_{\mcG_1}} - \mu_{\ba^*_{\mcG_2}}\right)\right] \label{eqlowerbound_mid1_1_nn}\\
    &= \P_{\mcG_1}(\mcE_{\pi}) \frac{T}{2}\left(\mu_{\ba^*_{\mcG_1}} - \mu_{\ba^*_{\mcG_2}}\right)  \label{eqlowerbound_mid2_nn}\\ 
    &=  \P_{\mcG_1}(\mcE_{\pi}) \frac{T}{2}   d^{\frac{L}{2}-1}  \delta \prod_{\ell=1}^{L} K^{(\ell)} \ ,
    \label{eq:lower_bound_regret_1_nn}
\end{align}
where \eqref{eqlowerbound_mid1_nn} lower bound the regret by considering the regret incurred when $a_1(t)>\frac{1}{2}$; \eqref{eqlowerbound_mid1_1_nn} holds due to $\mu_{\ba^*_{\mcG_2}}= \mu_{\ba}$ for all intervention $\ba\in\mcA_{\mcG_2}^*$  and we split the sum based on whether $\mcE_{\pi}$ holds; and \eqref{eqlowerbound_mid2_nn} uses the definition of $\mcE_{\pi}$, lower bounding for the occurrence time of regret by $\frac{T}{2}$ for $\mcE_{\pi}$, and by 0 for the $\mcE_{\pi}^{\C}$. The last equality in \eqref{eq:lower_bound_regret_1_nn} holds due to \eqref{eq:diff_mu_nn}. Similarly, for bandit instance $\mcG_2$, we have $\mu_{\ba^*_{\mcG_1}} = \mu_{\ba}$ for all intervention $\ba\in\mcA_{\mcG_1}^*$. Therefore, for $\mcG_1$, the expected value of instantaneous regret is $\mu_{\ba^*_{\mcG_2}} - \mu_{\ba^*_{\mcG_1}}$ if $a_i(t)\leq \frac{1}{2}$. Hence, we can similarly show that
\begin{align}
    \E_{\mcG_2}[R(T)] \geq  \P_{\mcG_2}(\mcE^{\rm c}_{\pi}) \frac{T}{2}   d^{\frac{L}{2}-1}  \delta \prod_{\ell=1}^{L} K^{(\ell)}  \ . \label{eq:lower_bound_regret_2_nn}
\end{align}
Combining~\eqref{eq:lower_bound_regret_1_nn} and~\eqref{eq:lower_bound_regret_2_nn}, we have
\begin{equation}
    \E_{\mcG_1}[R(T)] + \E_{\mcG_2}[R(T)] \geq \left(\P_{\mcG_1}(\mcE) + \P_{\mcG_2}(\mcE^{\rm c})\right) \frac{T}{2}  d^{\frac{L}{2}-1} \delta \prod_{\ell=1}^{L} K^{(\ell)} \ .
\end{equation}

Next, we use the Lemma~\ref{th:pinsker} and find a lower bound on the term $\P_{\mcG_1}(\mcE_{\pi})+ \P_{\mcG_2}(\mcE_{\pi}^{\C})$, resulting in
\begin{align}
    \E_{\mcG_1}[R(T)] + \E_{\mcG_1}[R(T)] \geq \frac{T}{4}   d^{\frac{L}{2}-1}  \delta \exp(-{\rm D}_{\rm KL}(\P_{\mcG_1} \kl  \P_{\mcG_2})) \prod_{\ell=1}^{L} K_{\ell} \ . \label{eq:lower_bound_regret_mid3_nn}
\end{align}
We provide the following lemma to characterize the KL-divergence between bandit instances $\mcG_1$ and $\mcG_2$.
\begin{lemma}
\label{lem:klnn}
    For causal bandit instance $\mcG_1$ and $\mcG_2$, the KL-divergence between $\P_{\mcG_1}$ and $\P_{\mcG_2}$ is upper bounded by
    \begin{equation}
    \label{eq:DL_bound_nn}
        {\rm D}_{\rm KL}(\P_{\mcG_1} \kl \P_{\mcG_2}) \leq T \log\left(\frac{1}{(1+2\delta)(1-2\delta)}\right) \ .
    \end{equation}
\end{lemma}
\begin{proof}
The proof follows the same steps as that of Lemma~\ref{lem:kllinear} with the difference that we split the summation in \eqref{eq:lowerlienr_mid_split} by event $\{a_i(t)\leq \frac{1}{2}\}$ and $\{a_i(t) > \frac{1}{2}\}$. 
\end{proof}
If we choose $\delta=\frac{1}{\sqrt{T}}$ to balance the terms in the lower bound, we obtain
\begin{align}
    \max\{\E_{\mcG_1}[R(T)],\E_{\mcG_2}[R(T)]\} &\geq \frac{1}{2} \left(\E_{\mcG_1}[R(T)] + \E_{\mcG_2}[R(T)]\right)\\
     &\overset{\eqref{eq:lower_bound_regret_mid3_nn}}{\geq} \frac{T}{8}  d^{\frac{L}{2}-1}  \delta \exp(-{\rm D}_{\rm KL}(\P_{\mcG_1} \kl  \P_{\mcG_2}))  \prod_{\ell=1}^{L} K^{(\ell)} \\
     &\overset{\eqref{eq:DL_bound_nn}}{\geq} \frac{T}{8} d^{\frac{L}{2}-1} \delta  [(1+2\delta)(1-2\delta)]^{T} \prod_{\ell=1}^{L} K^{(\ell)}  \\
     & =\frac{1}{8} d^{\frac{L}{2}-1}  \sqrt{T} \times \Big(1-\frac{4}{T}\Big)^{T} \prod_{\ell=1}^{L} K^{(\ell)}   \ .
\end{align}
Similarly to the previous two models, by setting $c=\frac{1}{8}\times 0.2^5$, we have 
\begin{equation}
     \max\{\E_{\mcG_1}[R(T)],\E_{\mcG_2}[R(T)]\} 
     \geq c K d^{\frac{L}{2}-1}\sqrt{T} \ .
\end{equation}

Hence, the policy $\pi$ incurs a regret $c K d^{\frac{L}{2}-1} \sqrt{T}$ in at least one of the two bandit instances.


\section{Further Comparisons with~\citep{sussex2022model}}
\label{sec:compareGP}

In this section, we analytically compare the upper bounds presented in this paper, and those of the model-based causal Bayesian optimization (MCBO) approach presented in~\citep{sussex2022model}. We present two comparisons. 
\begin{itemize}
    \item \textbf{Scaling with $N$.} The important distinction pertains to the scaling rate of the regret terms with respect to the graph size. Specifically, we discuss that the achievable regret of MCBO scales \emph{linearly} with graph size $N$, while as we discussed earlier, the regret of GCB algorithms scales with $N$ according to $\sqrt{\frac{\log N}{\log T}}$, which for finite $T$ is substantially slower than the linear growth rate, and diminishes as $T$ grows. 
    \item \textbf{Scaling with $T$.} We also provide a class of functions $\mcF$ for which the GCB algorithm is guaranteed to have sublinear regret in terms of $T$, while the MCBO algorithm is not. 
\end{itemize}
Since the regret analyses of the two methods are carried out via entirely distinct methods, we need to find a way to unify and interpret the results. More specifically, our results are presented in terms of the eluder dimensions and the covering numbers of function classes. On the other hand, the results of MCBO for Gaussian processes are presented in terms of \emph{information gain} associated with the corresponding reproducing kernel Hilbert spaces. To facilitate comparisons, we will provide a unifying way of presenting both results in terms of information gain measures, which is the basis for MCBO's analysis. We compare these results for a specific model, described next, and specialize the regret bounds of the GCB and MCBO algorithms for this model.

\noindent\textbf{RKHS with polynomial spectral decay.} For a given vector $\theta\in\R^{d_i}$, scalar $a_i\in\mcA_i$, and RKHS $\mcH_i$ with respect to a known kernel function $k_i$, set
\begin{equation}
        f_i(X_{\Pa(i)}\; ;\; a_i)= \inner{\theta}{k_i(\cdot,[X_{\Pa(i)}, a_i])}_{\mcH}\ ,
    \end{equation}
and accordingly, set
\begin{align}\label{eq:RKHS}
    \mcF_i = \{f_i\mid \norm{\theta}\leq 1, k_i \text{ satisfies Assumption~\ref{ass:decay}}\}\ .
\end{align}
\begin{assumption}
\label{ass:decay}
     (Polynomial decay). Denote the space of inputs to $f_i$ by $\mathcal{Z}_i\in \mathbb{R}^{d_i}\times \mcA_i$. Assume that the kernel function $k_i$ admits the following orthonormal decomposition such that for any $z, z' \in \mcZ_i$:
\begin{equation}
k_i\left(z, z^{\prime}\right)=\sum_{j=1}^{\infty} \lambda_j \phi_j(z) \phi_j\left(z^{\prime}\right)\ ,
\end{equation}
where $\{\lambda_j: j\in\N\}$ are countable and positive eigenvalues with corresponding eigenfunctions  $\{\phi_j : j\in\N\}$ that satisfy  $\left\|\phi_j(z)\right\|_{\infty} \leq C_\phi$ for $z\in\mcZ_i$ and the eigenvalues $\lambda_j$ have a  $\beta$-polynomial decay, i.e.,  $\lambda_j \leq C_p j^{-\beta}$, for some $\beta>2+\frac{1}{d}$ and constant $C_p$.  We show that for the RKHS class specified in~\eqref{eq:RKHS}, MCBO is not guaranteed to have a sublinear regret bound in $T$, whereas GCB maintains a sub-linear regret. 
\end{assumption}

\subsection{Information Gain}
We provide the definitions of information gain and the associated maximum and critical information gains customized for Gaussian processes~\citep{srinivas2009gaussian}.

\begin{definition}(Information Gain). Corresponding to node $i\in[N]$, consider the RKHS $\mcH_i$ with input space $\mcZ_i\triangleq \R^{d_i}\times \mcA_i$ and an associated kernel $k_i$. For any $\lambda>0$, the information gain upon observing $\{z_\ell\in \mcZ_i: \ell\in[T]\}$ is defined as
\begin{equation}
\gamma_i\left(\lambda ; \{z_1, \ldots, z_T\}\right)\triangleq \log \operatorname{det}\left(I+\frac{1}{\lambda} \sum_{i=1}^T z_i z_i^{\top}\right) \ ,
\end{equation}
where $I$ is the identity matrix.
\end{definition}

 Based on the notion of information gain, the maximum information gain is defined as follows, which captures the maximum volume of the ellipsoid generated by $T$ points in $\mcZ_i$.
\begin{definition}(Maximum Information Gain). For any $\lambda>0$ and $T \in\N$, the maximum information gain upon $T$ observations is defined as
\begin{equation} 
\gamma_{i,T}(\lambda) \triangleq \max _{\{z_\ell:\ell\in[T]\} \in \mcZ^T_i} \gamma_i\left(\lambda ; \{z_1, \ldots, z_T\}\right) \ .
\end{equation}
\end{definition}
Accordingly, we define
\begin{align}
    \gamma_T\triangleq \max_{i\in[N]} \gamma_{i,T}(\lambda)\ .
\end{align}
Finally, the following critical information gain is the minimum $T$ for which maximum information gain cannot grow linearly.
\begin{definition}(Critical Information Gain). For a fixed constant $c>0$, the critical information gain is defined as
\begin{equation}
\tilde{\gamma}_i(\lambda, c) \triangleq \min \left\{k \;:\;  \gamma_{i,k}(\lambda) \leq c k\right\} \ .
\end{equation}
\end{definition}

\subsection{Achievable Regret of the MCBO Algorithm}
\begin{lemma}[informal]~\cite[Theorem 1]{sussex2022model} Consider SCMs in which functions $f_i$ are Lipschitz-continuous and bounded in $\mcH_i$, noise terms are bounded, the Gaussian process space that satisfies the continuous kernel and calibrated assumption. Then, for all $T \geq 1$, with probability at least $1-\delta$, the cumulative regret of MCBO is bounded by
 \begin{equation}\label{eq:MCBO_regret}
    \E [R(T)] = \mathcal{O}\left(K \beta_T^L d^L N \sqrt{T \gamma_T}\right) \ ,
\end{equation} 
where $K$ is associated with Lipschitz constants of the function and kernel space, $\beta_T$ is the confidence radius, and $\sigma^2$ is the noise variance.
\end{lemma}
While the assumptions adopted in \citep{sussex2022model} differ from ours and are not directly comparable, this upper bound leads us to draw two primary conclusions.
\begin{itemize}
    \item First, the regret bound is linearly proportional to the number of nodes $N$.
    \item Secondly, the radius of the confidence set associated with estimating the parameters of the Gaussian process (mean and variance) is  
    \begin{align} \label{eq:GP_beta}
\beta_T=\tilde{\mathcal{O}}\left(\frac{\sqrt{\gamma_{T}}}{d} \right)\ ,
    \end{align}
    which hinges on the maximum information gain. By comparing \eqref{eq:MCBO_regret}~and~\eqref{eq:GP_beta}, we have the following growth for the regret:
\begin{align}
\E[R_T]=\mcO\left(\gamma_T^{\frac{L+1}{2}}\sqrt{T}\right)\ .    
\end{align}
This regret bound becomes linear in $T$ when $\gamma_T$ has super-logarithmic growth with $T$ and $L$ is large enough. Specifically, we show this happens for the class specified in~\eqref{eq:RKHS}, which we establish next. 
\end{itemize}

\begin{lemma}~\citep[Lemmas D.5]{yang2020function} \label{lem:boundinformationgain}
For the class of function $\mcF_i$ defined in~\eqref{eq:RKHS}, the maximum information gain is upper bounded as follows.
    \begin{equation}
    \gamma_{i,T}(\lambda) = \widetilde{\mathcal{O}}\left(\lambda^{-\frac{1}{\beta} }T^{\frac{d+1}{\beta+d}}\right)\ .
\end{equation}
\end{lemma}

We conclude that the regret of MCBO scales as
\begin{equation}
\label{equ:MCBO}
\E[R(T)] = \widetilde\mcO\left(T^{\frac{(L+1)(d+1)}{2(\beta+d)}+\frac{1}{2}}\right)\ .
\end{equation}
Thus, when $L\geq \frac{\beta+d}{d+1}-1$, the MCBO algorithm cannot guarantee sub-linear regret.

\subsection{Achievable Regret of the GCB Algorithm}
Next, we analyze the regret of the GCB algorithms for the same model analyzed in the previous section, i.e., the class of functions specified in~\eqref{eq:RKHS}. We leverage the following result on the eluder dimension and the critical information gain.
\begin{lemma}~\citep[Theorem 3.1]{huang2021short}
Consider the function class $\mathcal{F}_i$ defined in \eqref{eq:RKHS}. By setting $\lambda=\frac{\epsilon^2}{4}$, we have the following connection between the eluder dimension and the critical information gain:
\begin{equation}
    \operatorname{dim}(\mcF_i, \epsilon)< \tilde{\gamma}_i\left(\lambda, \log \frac{3}{2}\right)\ .
\end{equation}
\end{lemma}
Furthermore, from Lemma~\ref{lem:boundinformationgain} we know that
$\gamma_{i,T} = \widetilde{\mathcal{O}}\left(T^{\frac{d+1}{ \beta+d}}\right)$. Hence, the critical information gain satisfies~\citep{huang2021short}:
\begin{equation}
    \tilde{\gamma}_i(\lambda, c)= \widetilde{\mathcal{O}}\left(\lambda^{-\frac{\beta+d} {\beta(\beta-1)}}\right)\ .
\end{equation}
In our analysis and design of the algorithm in\eqref{equ:alpha_i}, we have set $\epsilon\geq \frac{1}{T}$, based on which we obtain
\begin{equation}\label{eq:GCB_dim_UP}
     \operatorname{dim}(\mcF_i, \frac{1}{T}) < \tilde{\gamma}_i\left(\frac{1}{4T^2}, \log \frac{3}{2}\right) = \widetilde \mcO\left( T^{\frac{2(\beta+d)}{\beta(\beta-1)}}\right)\ .
\end{equation}
Additionally, we have the following upper bound for the covering number.
\begin{lemma}~\citep[Lemmas D.2]{yang2020function}
Consider the function class $\mathcal{F}_i$ defined in \eqref{eq:RKHS}. The covering number of this class satisfies:
\begin{equation}
    \log \cn_{\alpha}(\mcF_i) = \widetilde{\mathcal{O}}\left(\epsilon^{-\frac{2}{ \beta-1}}\right)\ .
\end{equation}
\end{lemma}
Thus, by setting $\alpha = \frac{1}{T}$,  we obtain
\begin{equation}\label{eq:GCB_cn_UB}
    \log \cn_{\frac{1}{T}}(\mcF_i) = \widetilde \mcO \left( T^{\frac{2}{\beta-1}}\right) \ .
\end{equation}

Therefore, based on~\eqref{eq:GCB_dim_UP} and \eqref{eq:GCB_cn_UB}, we find the following bound for the GCB algorithms.
\begin{corollary}
   Consider the function class $\mathcal{F}_i$ defined in \eqref{eq:RKHS}. The regret upper bounds in Theorem~\ref{thm:regret} becomes
    \begin{equation}\label{eq:GCB_RKSH}
        \E[R(T)] = \widetilde\mcO\left( T^{\frac{2\beta+d}{\beta(\beta-1)}+\frac{1}{2}}\right)\ .
    \end{equation}
\end{corollary}
We observe that regret is independent of $L$. This provides the freedom to choose the values of other parameters and $L$ such that we maintain a sublinear growth in $T$ for the regret in~\eqref{eq:GCB_RKSH}. On the other hand,  having a linear growth for the regret in~\eqref{equ:MCBO} is not guaranteed. Specifically, this can be readily verified by setting $\beta=d=10$, for which ~\eqref{eq:GCB_RKSH} is sublinear whereas~\eqref{equ:MCBO} becomes linear for $L\geq 2$.

\section{Experiments}
\label{sec:appendix:experiments}
In this section, we assess the regret performance (average cumulative regret) of the GCB algorithm and its scaling behavior with the time horizon, and the graph parameters $d$ (graph maximum in-degree) and $L$ (maximum causal path length). We assess the performance of the the GCB-TS algorithm on polynomial and neural network SCMs.

\paragraph{Causal graph.} We consider a hierarchical graph as illustrated in Figure~\ref{fig:he_example}. This graph consists of $L+1$ layers with the first $L$ layers having $d$ nodes. The nodes between two adjacent layers are fully connected. Finally, the last layer consists of one node (reward node) that is fully connected to nodes in layer $L$.

\begin{figure}[h]
        \centering
        \includegraphics[height=1.8 in]{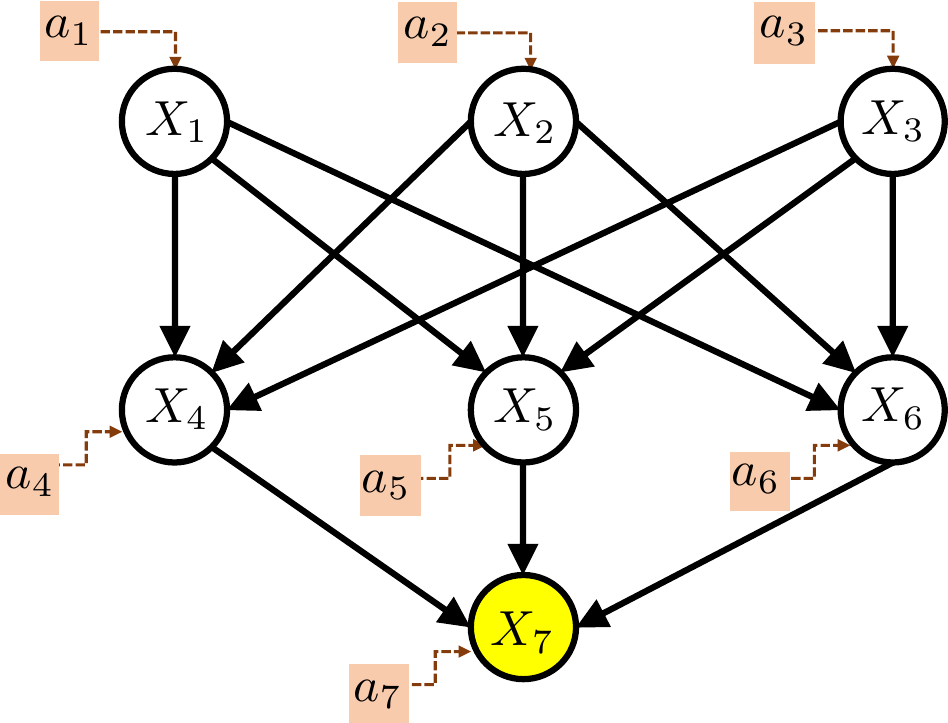}
        \caption{Hierarchical graph with  degree $d=3$ and causal length $L=2$.}
        \label{fig:he_example}
\end{figure}

\paragraph{Parameter setting.} In all experiments, we assume that the noise terms $\{\epsilon_i:i\in[N]\}$ are drawn from a standard Gaussian distribution $\mcN(0,1)$. To approximate the Bayesian, we assume that the parameters $\theta_i$ for $f_i$ have a Gaussian prior distribution and sample them from the Gaussian prior with small variance at each trail.
\begin{itemize}
    \item {\bf Polynomial SCMs.} For the causal bandit with polynomial SCMs, we consider the case when $p=2$, and set the mean of weights' Gaussian prior to $\frac{1}{d+1}$ to control the scaling behavior of the system. We set the intervention space to $\mcA_i=\{0,1\}$. 
    \item {\bf Neural network SCMs.} For the neural network SCMs, we consider the function space of a two-layer neural network with a leaky ReLU activation function with a slope of $\alpha=0.1$ and set the mean of weights' Gaussian prior to $\frac{1}{2\sqrt{d+1}}$. Furthermore, We set the intervention space as $\mcA_i=[0,1]$. 
\end{itemize}
\noindent\textbf{Algorithm settings.} For polynomial SCMs, the least square estimator $\tilde f_{i,t}$ in \eqref{equ:LS_estimates} can be solved by solving a linear system, and we sample $\bef$ from the posterior $\bar f_{i,t} \sim \mcN( \tilde f_{i,t}, \bar V_{i,t}^{-1})$, where $\bar V_{i,t}^{-1}$ is a Gram matrix related to the linear system. For neural network SCMs, we use stochastic gradient descent (SGD) to approximate the least square estimator $\tilde f_{i,t}$  defined in \eqref{equ:LS_estimates}, and we sample $\bar\bef$ from the posterior $\bar f_{i,t} \sim \mcN(\tilde f_{i,t}, V_{i,t}^{-1})$ with the Gram matrix $V_{i,t}$ defined as
\begin{equation}
    V_{i,t} = \sum_{s=1}^{t} X_{\Pa(i)}(s) X_{\Pa(i)}^{\top}(s) \ .
\end{equation}
We note that a similar sampling strategy is adopted by \citep{jun2017scalable, kveton2020randomized} for the generalized linear bandit.

\noindent\textbf{Regret performance.}
In Figure~\ref{fig:regret_poly} and Figure~\ref{fig:regret}, we present the cumulative regret of GCB-TS algorithm applied to polynomial SCMs and neural network SCMs, respectively. The underlying causal graph is set to be a hierarchical graph with maximum causal path length $L=2$. Each figure displays the cumulative regret for both $d=2$ and $d=3$.
It is noticeable that in scenarios with higher complexity, specifically the polynomial SCM of degree $d=3$, the regret scaling rates fall within the regime $\mcO(\sqrt{T})$, which is in line with our theoretical results. 

\begin{figure*}[h]
    \centering
    \begin{minipage}{.45\textwidth}
        \centering
        \includegraphics[width=6 cm]{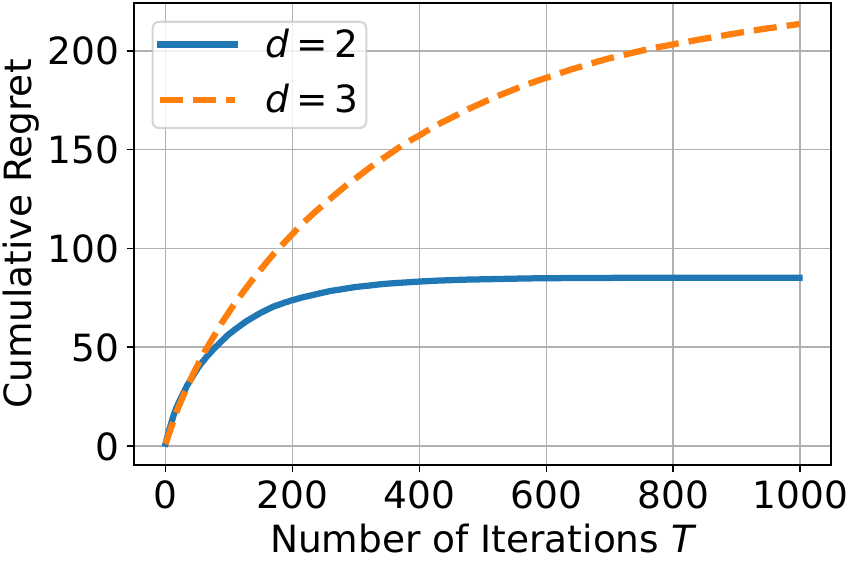}
        \caption{Cumulative regret vs $T$ (polynomial)}
        \label{fig:regret_poly}
    \end{minipage}
    \hspace{0.05\textwidth}
    \begin{minipage}{.45\textwidth}
        \centering        
        \includegraphics[width=6 cm]{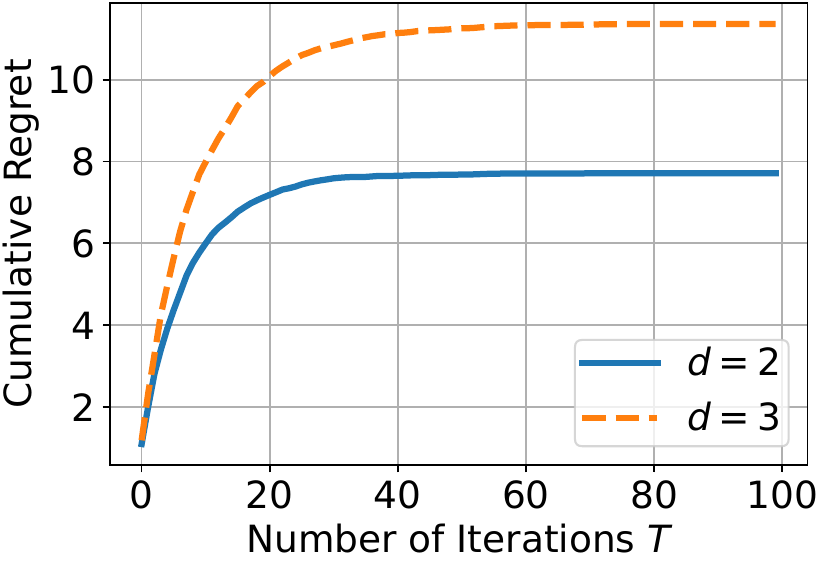}
        \caption{Cumulative regret vs $T$ (neural network)}
        \label{fig:regret}
    \end{minipage}
    \hspace{0.03\textwidth}
\end{figure*}

\noindent\textbf{Scaling behavior with graph parameters.} Figure~\ref{fig:digree_qua} and Figure~\ref{fig:digree} illustrate the variations of the cumulative regret versus the graph degree $d$ with the maximum causal path length $L=2$ under polynomial and neural network SMCs. Furthermore, Figure \ref{fig:digree_qua_1} and Figure \ref{fig:digree_1} display the corresponding variations when $L=1$. In these figures, we have also included the theoretical lower and upper bounds. In all four plots its is observed that the the scaling trend of the achievable regret are closer to the lower bound than the upper bounds. More specifically, the results in these figures indicate a polynomial behavior in $d$, which conforms to our theoretical results. 

\begin{figure*}[h]
    \centering
    \begin{minipage}{.45\textwidth}
        \centering
        \includegraphics[width=6 cm]{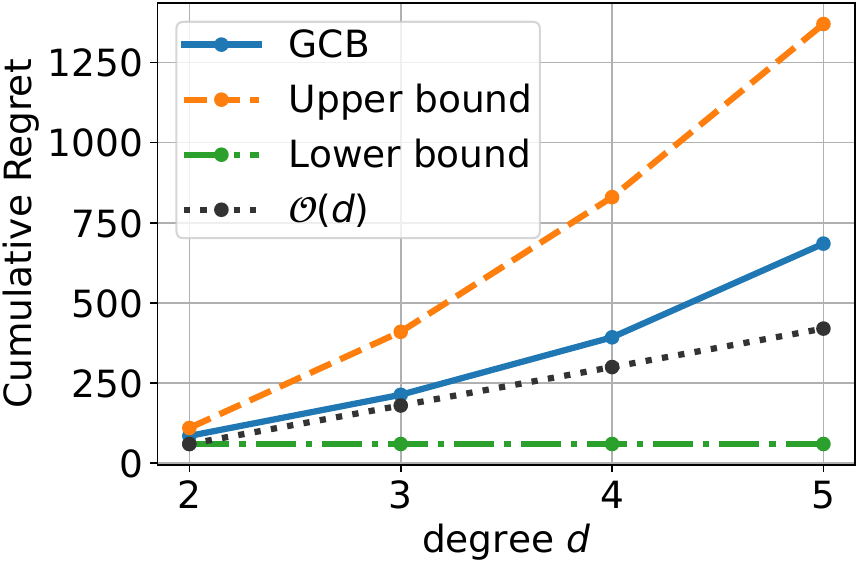}
        \caption{Cumulative regret vs degree $d$ for $L=2$ (polynomial)}
        \label{fig:digree_qua}
    \end{minipage}
    \hspace{0.05\textwidth}
      \begin{minipage}{.45\textwidth}
        \centering
        \includegraphics[width=6 cm]{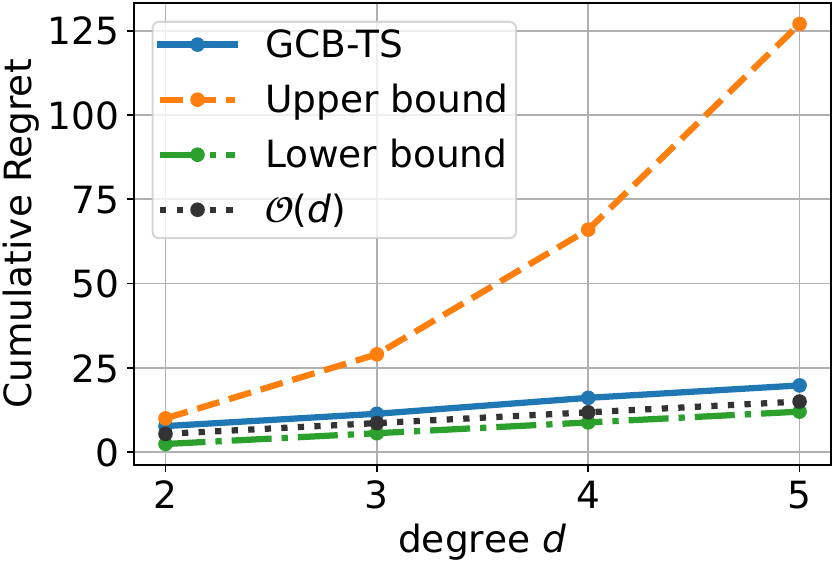}
        \caption{Cumulative regret vs degree $d$ for $L=2$ (neural network)}
        \label{fig:digree}
    \end{minipage}
\end{figure*}

\begin{figure*}[h]
    \begin{minipage}{.45\textwidth}
        \centering
        \includegraphics[width=6 cm]{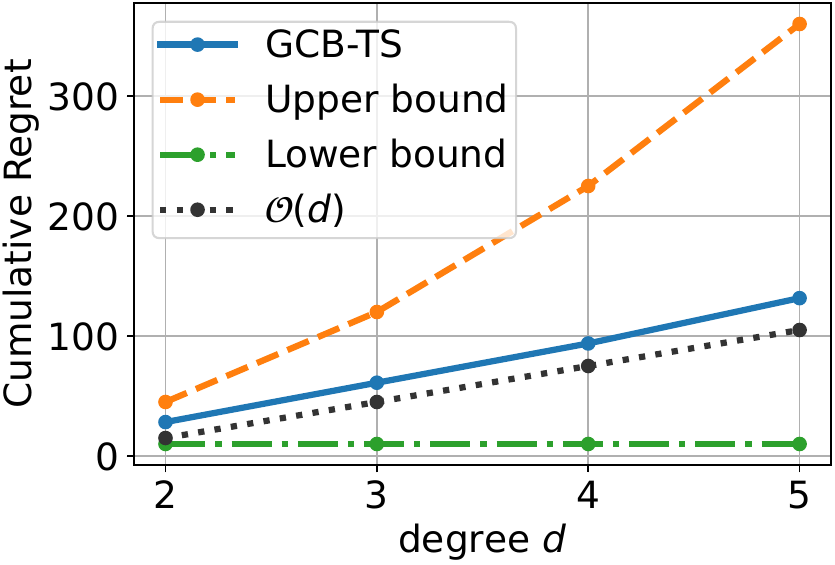}
        \caption{Cumulative regret vs degree $d$ for $L=1$ (polynomial)}
        \label{fig:digree_qua_1}
    \end{minipage}
    \hspace{0.05\textwidth}
    \begin{minipage}{.45\textwidth}
        \centering
        \includegraphics[width=6 cm]{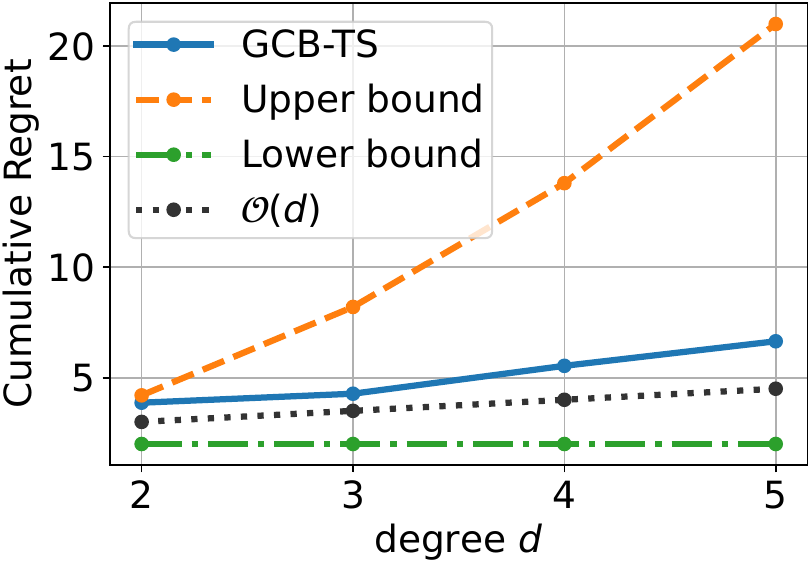}
        \caption{Cumulative regret vs degree $d$ for $L=1$ (neural network)}
        \label{fig:digree_1}
    \end{minipage}
\end{figure*}

Finally, Figure~\ref{fig:length_qua} and Figure~\ref{fig:length} elucidate the variations of the cumulative regret versus the maximum causal path length $L$, given a constant degree of $d=2$. The emerging patterns across these settings imply an exponential growth of regret with $L$, conforming with our theoretical findings. Notably, in the polynomial case, the scaling behavior of $L$ is considerably closer to that of the upper bound, attributed to the complexity of polynomial function spaces with growing value magnitude.

\begin{figure*}[h]
    \begin{minipage}{0.5\textwidth}
        \centering
    \includegraphics[width=6 cm]{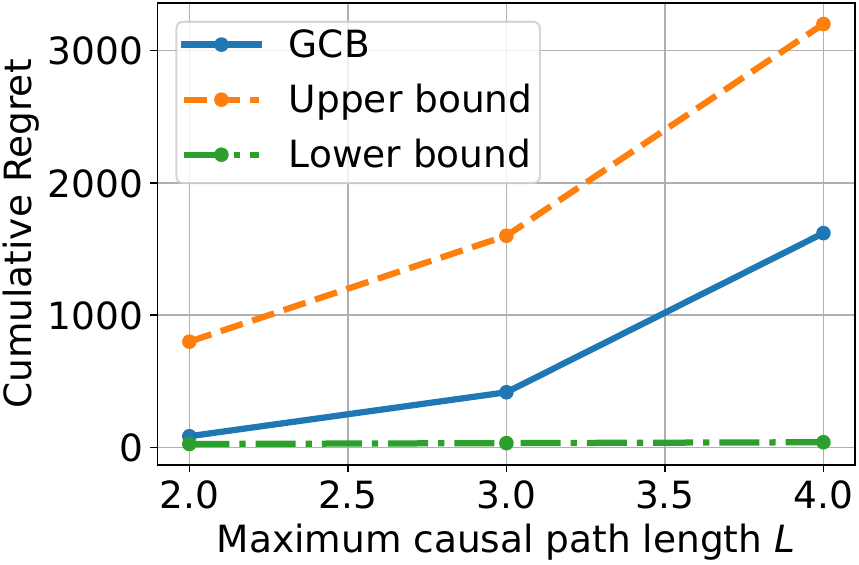}
        \caption{Cumulative regret vs $L$ for $d=2$\\ (polynomial)}
        \label{fig:length_qua}
    \end{minipage}
       \begin{minipage}{0.45\textwidth}
        \centering
    \includegraphics[width=6 cm]{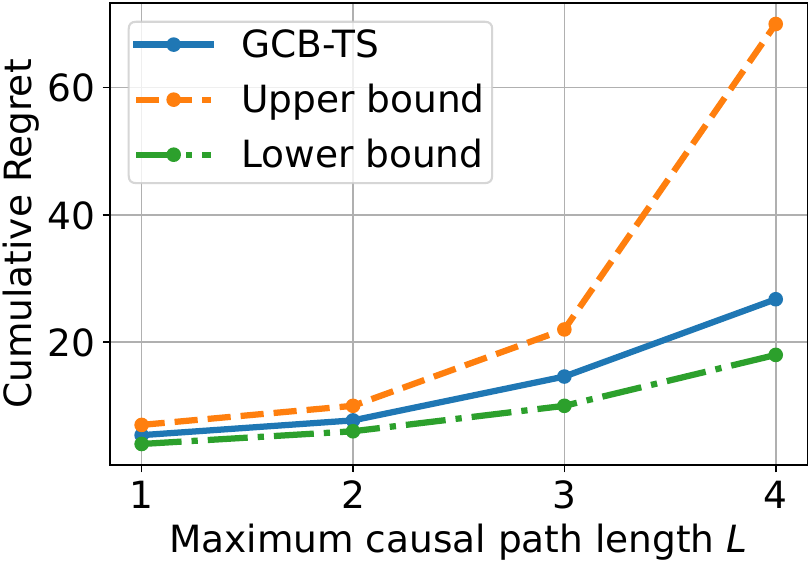}
        \caption{Cumulative regret vs $L$ for $d=2$\\ (neural network)}
        \label{fig:length}
    \end{minipage}
\end{figure*}

\begin{figure*}[h]
    \begin{minipage}{0.45\textwidth}
        \centering
    \includegraphics[width=6 cm]{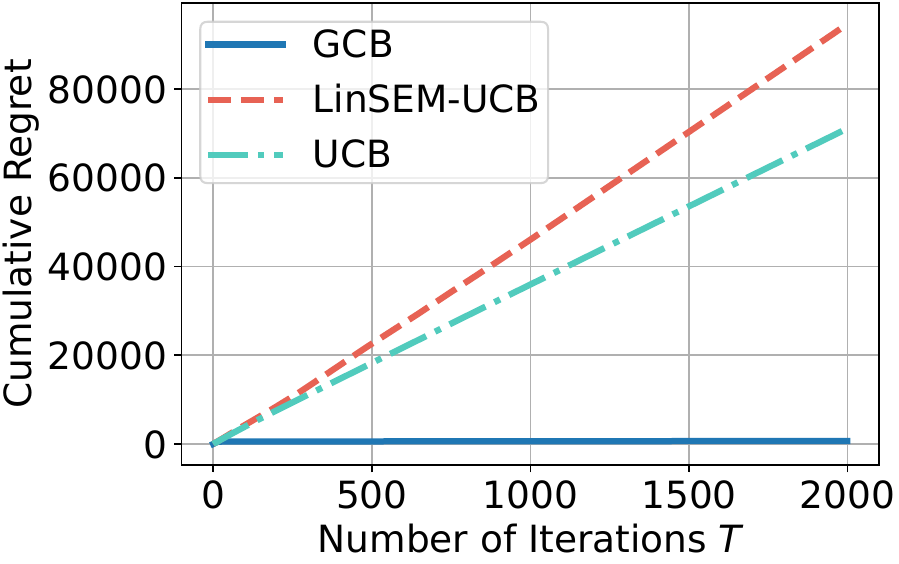}
        \caption{Cumulative regret for different \\ algorithms (polynomial)}
        \label{fig:Qua}
    \end{minipage}
       \begin{minipage}{0.45\textwidth}
        \centering
    \includegraphics[width=6 cm]{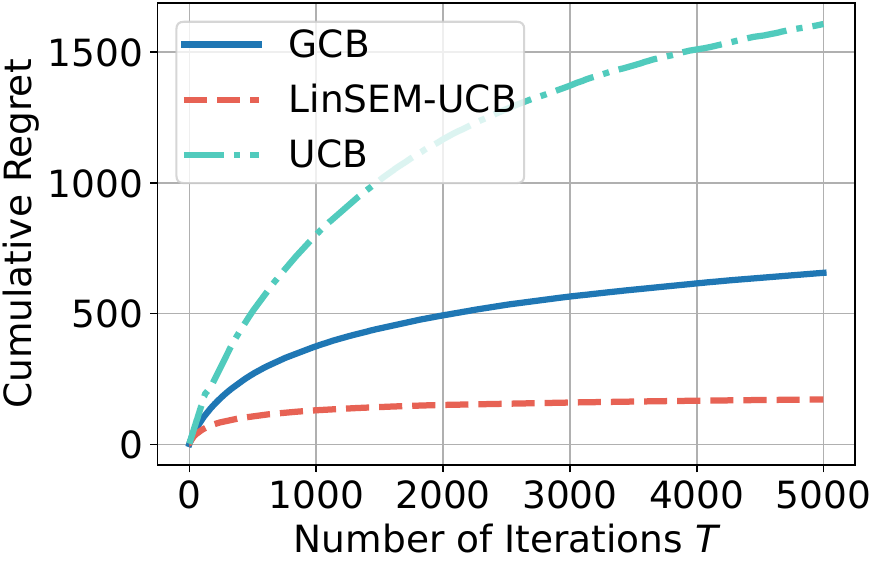}
        \caption{Cumulative regret for different \\ algorithms (linear)}
        \label{fig:Lin}
    \end{minipage}
\end{figure*}

\textbf{Comparison with baseline algorithms.} As the two most relevant existing approaches, we have compared the regret of our algorithm to those of LinSEM-UCB \citep{varici2022causal} and UCB \citep{agrawal1995sample}. In Figure~\ref{fig:Qua}, we provide the cumulative regret under a quadratic SCM. Both LinSEM-UCB and UCB algorithms yield linear regrets. LinSEM-UCB encounters failure due to a violation of the linearity of the system, and UCB fails because the reward node's distribution is heavy-tailed, which is difficult to estimate without information about the causal system. In Figure~\ref{fig:Lin},  we compare the results for linear SCMs. LinSEM-UCB, designed and optimized for this case, exhibits the lowest regret, while our algorithm shows a slightly higher regret. This discrepancy arises since our confidence radius is designed for general function classes, and its slightly weaker performance in this special setting is the cost of its universality. 


\section{Extension to $do$ interventions}
\label{sec:extenddo}
In this section, we extend the results we provided under soft interventions for $do$ interventions. Consider the same causal system as in Section~\ref{sec:GCB} with the distinction that we apply do $do$ interventions. Under applying $do$ intervention $a_i\in\mcA_i$ on node $i\in[N]$, we have $X_i=a_i$. We also use the convention of setting $a_i =0$ to generate the observational model on node $i$, under which the observation relationship is captured by $r_i(X_{\Pa(i)}) = f_i(X_{\Pa(i)}\;;\; 0) $. Hence, the SCMs under $do$ interventions can be restated as
\begin{equation}
    X_i =  \left\{
    \begin{array}{ll}
    r_i(X_{\Pa(i)})+\epsilon_i &\text{if } a_i=0 \;\; (\mbox{observational model)}\\
    a_i  &\text{if } a_i\neq 0 \;\; (\mbox{interventional model)}
    \end{array}\right. \ , \qquad \forall i\in[N] \ .
\end{equation}
where the functions $\{r_i:i\in[N]\}$ and are \emph{unknown} and $f_i$ and $r_i$ belong to a specified class $\mcF_i$  and $\mcR_i$, respectively.

\noindent\textbf{GCB Algorithms.} 
Our algorithm can be applied to this setting without modifications. The only impact of this change in the algorithms is in the estimators and confidence sets. Since the effects of the intervention are known to be constant, so we only need to estimate and construct confidence sets for the function $r_i \in \mathcal{R}_i$.

\noindent\textbf{Regret Analysis.} The regret guarantee in Theorem~\ref{thm:regret} and Corollary~\ref{cor:regret} still holds. For analyzing the regret guarantees in special SCMs, we provide the following relationships between the eluder dimensions and covering numbers of $\mcF_i$ and $\mcR_i$.
\begin{theorem}
    For a given function space $\mcF_i$ and the associated function space $\mcR_i$ defined for $do$ intervention,  we have the following relationships:
    \begin{align}
        \operatorname{dim}( \mcF_i,\alpha_i) & =  \operatorname{dim}(\mcR_i,\alpha_i) \ , \qquad 
       \mbox{and} \qquad   \cn_{\alpha_i}(\mcF_i)  =  \cn_{\alpha_i}(\mcR_i) \ .
    \end{align}
\end{theorem}

\begin{proof}
We first provide the relationship for eluder dimensions, ensured by the covering numbers.

\textbf{Eluder dimension.} 
We start the proof by showing that the longest ordered set in Definition~\ref{def:eluder} does not include the inputs $Z_i=(X_{\Pa(i)},a_i)$ with $a_i\neq 0$ when $\operatorname{dim}(\mcF_i,\alpha_i)\geq 2$. This is due to the fact that $Z_i$ is always $\alpha_i$-dependent on other inputs in $\mcZ_i$ since for any two functions $f,\tilde f \in\mcF_i$ we have
\begin{equation}
    f(X_{\Pa(i)} \f a_i) - \tilde f(X_{\Pa(i)} \f a_i) = a_i - a_i = 0 < \alpha_i \ .
\end{equation} 
Hence, it is sufficient to prove the next property. Corresponding to any given ordered set $\{Z_i(m):m\in[n]\}$, where $Z_i (m)= (X_{\Pa(i)}(m), a_i(m))$, we define the ordered set $\{\bar Z_i(m):m\in[n]\}$, where $\bar Z_i (m)= X_{\Pa(i)}(m)$. We show that the following two statements are equivalent for these two ordered sets.
\begin{itemize}
    \item \textbf{S1.} Each element  in $\{Z_i(m):m\in[n]\}$ is $\alpha_i$-independent of its predecessors with respect to $\mcF_i$.
    \item \textbf{S2.} Each element in $\{\bar Z_i(m):m\in[n]\}$ is $\alpha_i$-independent of its predecessors with respect to $\mcR_i$.
\end{itemize}
This is obvious due to the fact that for any $f\in\mcF_i$ and $r\in \mcR_i$ that satisfy $r(\cdot) = f(\cdot \f 0)$ we have
\begin{equation}
    r(\bar Z_i(m)) = f(Z_i(m)) \ .
\end{equation}
Thus, based on the above two claims, we conclude that $\operatorname{dim}(\mcF_i,\alpha_i) =  \operatorname{dim}(\mcR_i,\alpha_i)$.

\textbf{Covering number.} We prove this part by
showing that for any $\alpha_i$-cover of $\mcF_i$, there exists an $\alpha_i$-cover of $\mcR_i$ with an identical cardinality, and vice versa. For this purpose, for any $\mcC$ that is an $\alpha_i$-cover of $\mcF_i$, we construct $\bar \mcC$ as 
    \begin{equation}
        \bar\mcC = \big\{r\mid r=f(\cdot \f 0) \text{ for } f \in \mcC\big\} \ .
    \end{equation}
    Subsequently, for any $r\in \mcR_i$, we can construct $f \in \mcF_i$ as
    \begin{equation}
        f(\cdot \f a_i) = 
        \left\{
        \begin{array}{ll}
        r(\cdot)+\epsilon_i & a_i=0 \\
        a_i  & a_i\neq 0
    \end{array}\right. \ .
    \end{equation}
    Since $\mcC$ is an $\alpha_i$-cover of $\mcF_i$, we can find a function $\tilde f\in\mcC$ such that $\|f-\tilde f\|_{\infty}\leq \alpha_i$. Hence,  we find $\tilde r = \tilde f(\cdot \f 0)\in\bar\mcC$ such that
    \begin{equation}
        \norm{r-\tilde r}_{\infty} = \norm{f-\tilde f}_{\infty}  \leq \alpha \ .
    \end{equation}
        Thus, $\bar \mcC$ is an $\alpha_i$-cover of $\mcR_i$ with the same cardinality as $\mcC$ and we have
    \begin{equation}\label{eq:cn_ineq1}
        \cn_{\alpha_i}(\mcF_i) \geq \cn_{\alpha_i}(\mcR_i)  \ .
    \end{equation}
    On the other hand, for any $\bar\mcC$  that is an $\alpha_i$-cover of $\mcR_i$, if we construct $\mcC$ as 
    \begin{equation}
        \mcC = \left\{f \; \Big| \; f(\cdot,a_i)=\left\{\begin{aligned}
        &g(\cdot)+\epsilon_i &\text{if } a_i=0\\
        &a_i  &\text{Otherwise}
    \end{aligned}\right. \ , \;\;  \text{ for } r\in\bar\mcC \right\} \ .
    \end{equation}
    Then for any $f \in \mcF_i$, we can construct $g$ as $g = f(\cdot, 0)$. Since $\bar \mcC$ is an $\alpha_i$-cover of $\mcR_i$, there exists $\tilde r \in \bar \mcC$ such that $\norm{r-\tilde r}\leq \alpha_i$. Hence we find $\tilde f\in\mcC$ that satisfies
    \begin{equation}
        \tilde f(\cdot,a_i)=\left\{\begin{aligned}
        &\tilde r(\cdot)+\epsilon_i &\text{if } a_i=0\\
        &a_i  &\text{Otherwise}\end{aligned} \right.  \ .
    \end{equation}

    We know $\tilde f\in \mcC$ and
    \begin{equation}
        \norm{f-\tilde f}_{\infty} = \norm{r-\tilde r}_{\infty}  \leq \alpha_i \ .
    \end{equation}
Thus, $\mcC$ is an $\alpha_i$-cover of $\mcF_i$ with the same cardinality as $\bar \mcC$ and we have
    \begin{equation}\label{eq:cn_ineq2}
        \cn_{\alpha_i}(\mcF_i) \leq \cn_{\alpha_i}(\mcR_i) \ .
    \end{equation}
    Therefore, by comparing~\eqref{eq:cn_ineq1} and ~\eqref{eq:cn_ineq2} we conclude
    \begin{equation}
        \cn_{\alpha_i}(\mcR_i) = \cn_{\alpha_i}(\mcF_i)  \ .
    \end{equation}
\end{proof}

\end{document}